\documentclass[aps,prx,reprint,amsmath,amssymb,longbibliography]{revtex4-2}

\usepackage{graphicx}   
\usepackage{bm}         
\usepackage{amsfonts}
\usepackage{amsthm}     
\usepackage{comment}    
\usepackage{booktabs}   
\usepackage[dvipsnames,table]{xcolor}
\usepackage{colortbl}   
\usepackage{hyperref}   
\usepackage{dcolumn}
\usepackage{float}

\usepackage{booktabs}       
\usepackage{amsfonts}       
\usepackage{nicefrac}       
\usepackage{microtype}      

\usepackage{accents}
\usepackage{mathtools}
\usepackage{graphicx}
\usepackage{ragged2e}
\usepackage{subcaption}
\usepackage{siunitx}

\usepackage{tcolorbox}
\usepackage{multirow}
\usepackage{cancel}
\usepackage{wrapfig}
\usepackage{footmisc}
\usepackage[normalem]{ulem}

\usepackage{hyperref}
\definecolor{tumblue}{RGB}{0, 101, 189}
\hypersetup{
   colorlinks=true,
   citecolor=tumblue,
   linkcolor=tumblue,
   filecolor=tumblue,
   urlcolor=tumblue,
 }
 \usepackage[capitalize]{cleveref}
\usepackage{algorithm}
\usepackage{algpseudocode}
\algrenewcommand\algorithmicrequire{\textbf{Input:}}
\algrenewcommand\algorithmicensure{\textbf{Output:}}
\algnewcommand\Parameters{\item[\textbf{Parameters:}]}

\hypersetup{
  colorlinks=true,
  linkcolor=blue,
  urlcolor=cyan,
  citecolor=blue
}

\newtheorem{theorem}{Theorem}[section]

\newtheorem{definition}[theorem]{Definition}
\newtheorem{lemma}[theorem]{Lemma}

\DeclareMathOperator{\sgn}{sgn}
\DeclareMathOperator{\supp}{supp}
\DeclareMathOperator{\argmin}{arg\,min}

\begin{document}

\title{Learning with Boolean threshold functions}

\author{Veit Elser}
\thanks{ve10@cornell.edu}
\affiliation{Department of Physics, Cornell}
\author{Manish Krishan Lal}
\thanks{manish.krishanlal@gmail.com}
\affiliation{Mathematics, Independent Researcher, India}

\date{\today}

\begin{abstract}
We develop a method for training neural networks on Boolean data in which the values at all nodes are strictly $\pm 1$, and the resulting models are typically equivalent to networks whose nonzero weights are also $\pm 1$. The method replaces loss minimization with a nonconvex constraint formulation. Each node implements a Boolean threshold function (BTF), and training is expressed through a divide-and-concur decomposition into two complementary constraints: one enforces local BTF consistency between inputs, weights, and output; the other imposes architectural concurrence, equating neuron outputs with downstream inputs and enforcing weight equality across training-data instantiations of the network. The reflect–reflect–relax (RRR) projection algorithm is used to reconcile these constraints.

Each BTF constraint includes a lower bound on the margin. When this bound is sufficiently large, the learned representations are provably sparse and equivalent to networks composed of simple logical gates with $\pm 1$ weights. Across a range of tasks—including multiplier-circuit discovery, binary autoencoding, logic-network inference, and cellular automata learning—the method achieves exact solutions or strong generalization in regimes where standard gradient-based methods struggle. These results demonstrate that projection-based constraint satisfaction provides a viable and conceptually distinct foundation for learning in discrete neural systems, with implications for interpretability and efficient inference.
\end{abstract}

\maketitle

\section{Introduction}

This year marks the 40th anniversary of ``Learning representations by back-propagating errors''\cite{rumelhart1986learning}, the paper that launched the modern era of machine learning, known to the public as AI. Rumelhart, Hinton, and Williams understood that nonlinearities were required to break out of the more limited representations used by ``perceptrons'' \cite{rosenblatt1958perceptron}, and that gradient optimization of network parameters could still be carried out efficiently, thanks to the chain rule of calculus. Though technical in nature, most workers in machine learning would find it hard to imagine a science of automated learning that was not built on the foundation of back-propagation.

Just three years prior to back-propagation, an optimization breakthrough in a little known image processing task called \textit{phase retrieval} was making waves. Like artificial neural networks, phase retrieval works with large sets of data in the presence of nonconvexity. In neural networks, it is the function being minimized that is nonconvex, while in phase retrieval, it is the geometry of the constraint sets.
Though constraint-based methods---with projections replacing gradient steps---had been used in phase retrieval for years, Jim Fienup's ``hybrid-input-output'' algorithm \cite{fienup1982phase} was miraculous by comparison.
Whether neural network optimization or phase retrieval is the harder instance of nonconvexity is open to debate. In any case, it is interesting that the function minimization approach never gained traction in phase retrieval (with constraints expressed in terms of an error function).

The constraint formulation of neural networks is no less natural than the usual one based on loss functions. In this article, we imagine a hypothetical world where
Fienup’s breakthrough in constraint satisfaction for non-convex problems was so widely known it would have been considered alongside gradient-descent. Because there is considerably more freedom in setting up constraint formulations than writing down loss functions, much of how we understand the new technique was developed for diverse applications---puzzles, sphere packings---that have nothing to do with phase retrieval or neural networks. This article contains only a brief summary. For a more comprehensive review suitable for nonexperts we refer the reader to \cite{elser2025solving}.

This article is organized as follows. In the overview, we give a high-level comparison of the two approaches to network optimization, highlighting the advantages brought by the constraint-based approach. A key technical element for achieving interpretability and boosting generalization is the choice of nonlinear activation: the Boolean threshold function. In the review of BTFs that follows, it becomes clear that these are not ``just functions'' in the new setting. How the many constraints in neural network training are combined into just two easy constraints $A$ and $B$, using the divide-and-concur framework, is covered next. This is followed by a brief exposition of Fienup's algorithm and its generalization in geometric terms, where the work in each step (iteration) is performed by projections to the sets $A$ and $B$.

A series of six applications follows. The problem of inferring a small multiplier circuit from a multiplication table is a good way to get familiar with the new approach. A historically interesting application comes next, where we revisit the binary autoencoder and manage to get true binary codes. MNIST-related problems follow, because the introduction of a new method would be incomplete without them. The applications that deserve the most attention are the ones that demonstrate generalization: logic circuits and cellular automata. The generalization exhibited for these complex tasks is quite striking.

Readers eager to get a more hands-on experience can go to the software repository \texttt{boolearn} and perform all the experiments reported here. The software is entirely in C, uses only the standard libraries, and runs in the Linux environment.

\section{Overview}\label{sec:overview}

The function-minimizing approach to neural network optimization places restrictions on admissible activation functions. In order to compute gradients, these should at least be piecewise differentiable. The sigmoid and hyperbolic-tangent functions not only are smooth everywhere, but they additionally have the nice feature that they saturate to values that can be interpreted as \textsc{True/False}.

The binary character of sigmoids was mentioned in the technical report \cite{rumelhart1985learning} on which ``Learning representations by ...'' was based. Rumelhart et al. used an autoencoder network to drive home their optimization method's ability to find weights that could map $2^k$ linearly independent vectors to themselves, represented on $2^k$ nodes, even when forced to pass through a ``code-layer'' with only $k$ nodes. Nonlinear activation is essential for this parlor trick. The authors went on to remark that the $0/1$ saturations of their sigmoid functions meant their network was encoding into and decoding-from, a binary code. But as their results showed, their codes often included the value 1/2---the output of a sigmoid when it is furthest from being saturated.

The augmentation of the binary code can be overlooked when the application does not call on that level of interpretability. On the other hand, nice things are possible when the binary code is strict. The decoder-half of the autoencoder is a simple generative model, where $2^k$ data vectors are generated by sampling a distribution on the $k$ inputs. But the generative model only works when the admissible input-codes are a known code.

The easiest way to realize strict binary activation is through the Boolean threshold function:
\begin{equation}\label{eq:BTFformula}
y=\sgn(w\cdot x).
\end{equation}
This is the function used throughout this work. The weights $w$ are the parameters of the network, and are learned as in standard neural networks. But the output $y$ and the elements of the input vector $x$ are always $\pm 1$.

One way that Boolean threshold functions, or BTFs, are used differently here is that in addition to constraint \eqref{eq:BTFformula}, we also impose
\begin{equation}\label{eq:marginconstraint}
|w\cdot x|\ge\mu.
\end{equation}
The positive number $\mu$ is the BTF's \textit{margin}. Our BTFs are not functions in the usual sense. The margin constraint \eqref{eq:marginconstraint}, when satisfied with suitable weights for all the training data, compels all the inner products $w\cdot x$ to avoid a gray area set by $\mu$. In a sense, a successfully trained BTF
network makes decisions with conviction.

Implementing BTFs as just described, with a loss function, while not impossible, is complicated and requires additional parameters. In our approach, constraints \eqref{eq:BTFformula} and \eqref{eq:marginconstraint} are treated as a single ``BTF constraint'' and there is an elementary operation, called a projection, that takes an arbitrary trio $(w,x,y)\in\mathbb{R}^{m+m+1}$ and outputs the unique trio $(w',x',y')$ that satisfies the BTF constraint while minimizing
\begin{equation}\label{eq:BTFdistance}
\|w'-w\|^2+\|x'-x\|^2+(y'-y)^2.
\end{equation}
These projections are carried out concurrently on all the neurons/BTFs of the network. We call the independent, aggregate constraints at the BTFs, the $A$ constraint. Geometrically, $A$ is a continuous set in the high-dimensional space of $w$, $x$ and $y$ variables. When the network has $N$ nodes and $E$ edges, the dimension of that space is $N+2E$. The weights ($w$) and inputs ($x$) live on the edges because they are associated with sending/receiving node-pairs. 

The idea behind treating all the neurons independently, in the $A$ constraint, is called \textit{divide-and-concur} \cite{gravel2008divide}. There is another aggregate constraint, called $B$, that imposes ``concur'' or the equality of one neuron's output ($y$) with another neuron's input ($x$). Implementing the network architecture is strictly the job of the $B$ constraint. This constraint is a hyperplane (a system of many linear equations) and is even easier to project to than the $A$ constraint.

The weights also appear in the $B$ constraint, in a way that represents another departure from gradient-based optimization. Instead of there being one set of network variables (comprising $w$'s, $x$'s, and $y$'s), the constraint-approach uses as many instantiations of the network as there are data items. Each instantiation has its own set of $x$ and $y$ variables, since each data item constrains these differently at the network's inputs and outputs. But forcing the $w$ variables to be the same across all the instantiations would greatly complicate the projection to the $A$ constraint. Not only do we want $(w,x,y)$ to be independent across all the nodes/BTFs, we want that independence to extend across data items. The divide-and-concur solution is to have separate weights for each data item (to keep the $A$ constraint simple) and impose their equality in the $B$ constraint.

In addition to forcing the weights to concur across data items, the $B$ constraint imposes the constraint
\begin{equation}\label{eq:wnorm}
w\cdot w = m
\end{equation}
on a BTF's weights when it has $m$ inputs. This matches the normalization of the Boolean inputs, $x\cdot x = m$. And since $y=\pm 1$, all three variable types have rms value 1. Sensible variable normalizations are important if we use the distance \eqref{eq:BTFdistance} when computing projections.

Gradient-descent optimized networks are usually trained in the ``stochastic'' mode, where the gradient of the loss is evaluated for small, random batches of the data. When the entire data set is used, cancellation in the gradient computation can result in stagnation. There is no analogous problem in the constraint-based approach.

To understand the last statement, one needs to first learn how the constraint satisfaction algorithm, using projections, finds variable assignments that satisfy both $A$ and $B$. This is covered in Section \ref{sec:RRR}. For the present, it suffices to know that if $z_A$ and $z_B$
are projections of the full set of network variables---denoted by $z$---to the two constraints, then the optimizer moves $z$ by the rule
\begin{equation*}
z'=z+\beta(z_A-z_B)
\end{equation*}
in each iteration. The positive number $\beta$ is the ``time-step'' and is roughly analogous to the learning-rate parameter of gradient descent. But the difference $z_A-z_B$ is not the gradient of anything. Instead, the ``gap''
\begin{equation}\label{eq:gap}
\Delta=\|z_A-z_B\|
\end{equation}
measures the distance that is left to reconcile between the ``divide'' and ``concur'' constraints.

When the gap is zero, both constraints are satisfied. That is, all the individual BTF constraints are satisfied, for all the data items, and there is just one shared set of weights. In constraint problems where $A$ and $B$ are convex, $\Delta$ always decreases and behaves like the loss in gradient descent \cite{bauschke1996projection}. But in our case $A$ is not convex, and $\Delta$ can increase. When this happens, the algorithm is extricating itself from situations where the $A$ and $B$ variable assignments are locally irreconcilable. This nice feature persists when the training batches are large. So unlike stochastic-gradient-descent, which uses small batches, in the constraint-based method the batch size should be as large as the computing resources allow.

The margin \eqref{eq:marginconstraint} and weight normalization \eqref{eq:wnorm} constraints, acting together, bring about another nice feature: network sparsity. Like the strict Boolean activation of the BTF, networks with few relevant edges (having significant weight) bring interpretability to the representation. In Section \ref{sec:BTF} we prove that by defining the imposed margin on $m$-input BTFs as
\begin{equation}\label{eq:sigma}
\mu_m=\sqrt{m/\sigma}\;,
\end{equation}
where $\sigma$ is called the \textit{support hyperparameter}, then up to a caveat all the BTFs in the network will be equivalent to BTFs having at most $\sigma$ nonzero elements in their weight vectors. The caveat is that, in training, each BTF must see the complete set of possible input patterns on its ``relevant'' weights---those that are nonzero in the equivalent BTF. This condition has a reasonable chance of being satisfied when the BTF supports are small.

A corollary of the support property that we  impose through the margin constraint and \eqref{eq:sigma}, is the property that when $\sigma$ is an odd integer the only way that BTFs can satisfy the equality case of the margin constraint is when the weight vector has exactly $\sigma$ nonzero and equal-magnitude elements. The case $\sigma=3$, corresponding to weights
\begin{equation}
w=\sqrt{m/3}\,(\pm 1,\pm 1,\pm1,0,\;\ldots\;,0)
\end{equation}
and permutations, is already interesting. BTFs with such weights implement the 3-input majority gate (\textsc{Maj}), which acts as a 2-input \textsc{And} or \textsc{Or} when one of its inputs is held fixed. All the BTFs in our networks have an edge to a special node with value $-1$, so by choice of the weight to that node, the BTF can elect to implement a 2-input \textsc{And/Or}. The margin $\mu_m=\sqrt{m/3}$ also admits BTFs with fewer than 3 relevant input weights, but these are all equivalent to a BTF with just one relevant input---the \textsc{Copy}-gate.

\begin{figure}
    \centering
    \includegraphics[width=\columnwidth]{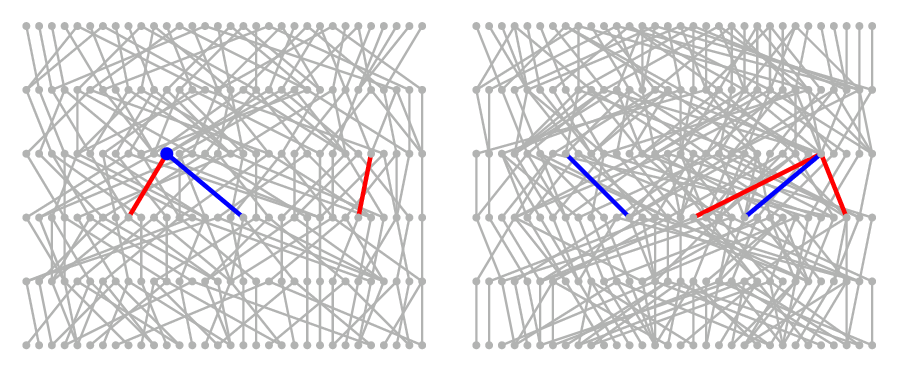}
    \caption{Two 5-layer random logic circuits where each node either copies a value from the layer below (isolated colored lines) or applies a logical gate to them (colored groups). The gates are either 2-input \textsc{And/Or} (left circuit) or 3-input \textsc{Maj} (right circuit). The colors (red/blue) of the highlighted edges indicate whether or not negation is applied. The blue node in the left circuit means this BTF is taking input from the constant node with weight $+1$, thereby choosing to implement \textsc{And}.}
    \label{fig:randlogic}
\end{figure}

There is no better demonstration of what the constraint-based method is capable of than the reconstruction of logic circuits (Fig. \ref{fig:randlogic}). Whereas the number of Boolean functions from $m$ inputs to $m$ outputs is astronomical ($2^{m 2^m}$), learning the function from a modest amount data is possible, in principle, if we know it can be expressed by a circuit comprising just \textsc{Not} and few-input logic gates arranged in a few-layer network. When deployed on a network with fully connected layers and the setting $\sigma=3$, our method is forced to consider only circuits where each BTF either just copies a Boolean value from one of the nodes in the layer below, or applies simple logic to 2 or 3 of those values (depending on whether the special constant node was chosen as one of the three allowed by $\sigma=3$). In any case, there is no information telling the constraint solver how one layer is sparsely wired to the next, or where to insert the \textsc{Not} gates (weights of negative sign).

\begin{figure}
    \centering
     \includegraphics[width=\columnwidth]{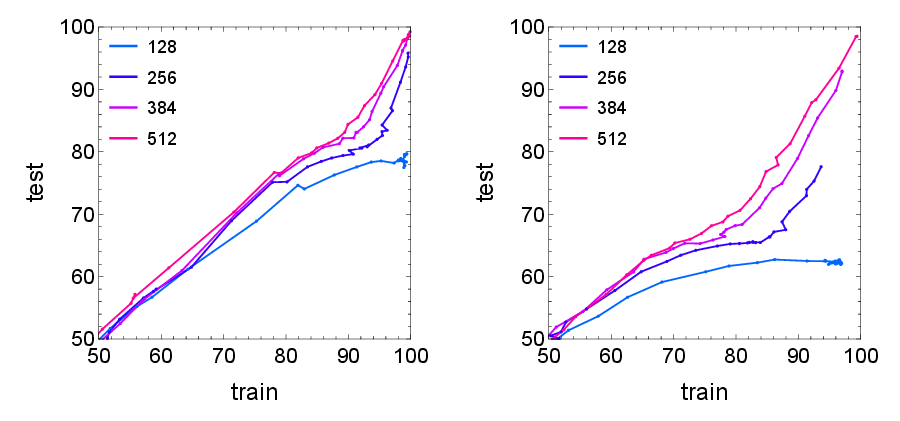}
    \caption{Evolution of the training and test accuracies when training on 128, 256, 384, and 512 data generated by the random logic networks above. Learning the random \textsc{And/Or} data (left plot) is easier, but 256 data appear to be sufficient to learn both types of data, given sufficient iterations of the algorithm.}
    \label{fig:accplots}
\end{figure}

While more details are given in Section \ref{sec:randlogic}, a preview of circuit reconstruction should convey that constraint-based learning with BTF networks is remarkable. Figure~\ref{fig:randlogic} is a rendering of two 5-layer circuits that were used to generate data---pairs of strings of 32 bits. The evolution of the training and test accuracies for different numbers of training data are plotted in Figure \ref{fig:accplots}. Accuracies are just the fraction of correct output bits when given an input string from the training set or, for testing, a string the learning algorithm has never seen. We see evidence of a transition at 256 training items. Below this number the reconstructed circuits are clearly different from those in Figure~\ref{fig:randlogic} because the test accuracy is well below 100\%. But above 256 training items the test accuracies for both data sets appear to be on track to reach 100\%, given enough iterations of the constraint solver (limited to $10^6$ in this demonstration).

Figure~\ref{fig:accGDplots} shows accuracy results for the same data when the gradient-descent method is used. All the results we report using this method should be seen as \textit{baselines}: starting points that use state-of-the-art optimiziation on the most widely used multilayer perceptron model. Details can be found in appendix~\ref{sec:gd_baselines}. Refinements of the baseline method clearly have a ways to go to equal the performance of the constraint-based method. The test accuracies in Figure~\ref{fig:accGDplots} show no signs of reaching 100\%, and there is no evidence of a transition at some number of training data. We estimate the information sufficiency for perfect circuit reconstruction in Section \ref{sec:randlogic} and find it is just above 128 items. While both methods were limited to $10^6$ steps/iterations, the plots suggest only the constraint method's test accuracy would improve when this number is increased.

\begin{figure}[t!]
    \centering
    \includegraphics[width=\columnwidth]{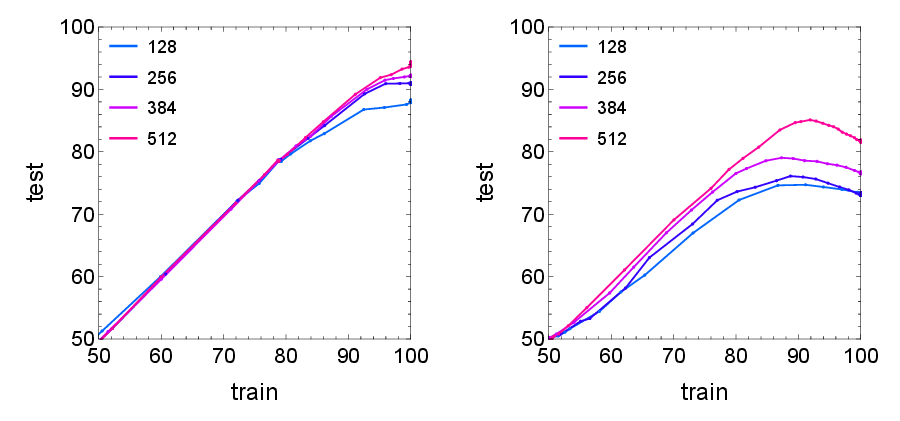}
        \caption{Same evolution of learning curves as in Figure \ref{fig:accplots} but for the gradient-descent method.}
    \label{fig:accGDplots}
\end{figure}

High test-accuracies, when faced with ``a poverty of stimulus,'' can be explained by the fact that the learning algorithm is forced to only consider models that are very close to the models that generated the data (same number of layers, sparse, etc.). That would be truly amazing and the use of the term ``reconstruction'' would be justified. However, the distribution of the learned weights shows that something not quite that simple is going on. These are shown for the  5-layer-circuit data in Figure \ref{fig:wgthist}. Because the weight vector for every neuron has normalization $w\cdot w=32+1$, when we see peaks in the distribution at $\pm\sqrt{33/1}\approx \pm5.7$ it means there are neurons receiving effectively just one input (\textsc{Copy}). Similarly, peaks at $\pm\sqrt{33/3}\approx \pm 3.3$ correspond to \textsc{And/Or/Maj} gates. The distributions in the lower panels show just the lowest-layer weights, and these are indeed close to the values just described. The upper panels show the weight distributions over all five layers, and these are more diverse. Is the algorithm getting sloppy in the higher layers?

\begin{figure}[b!]
    \centering
    \includegraphics[width=\columnwidth]{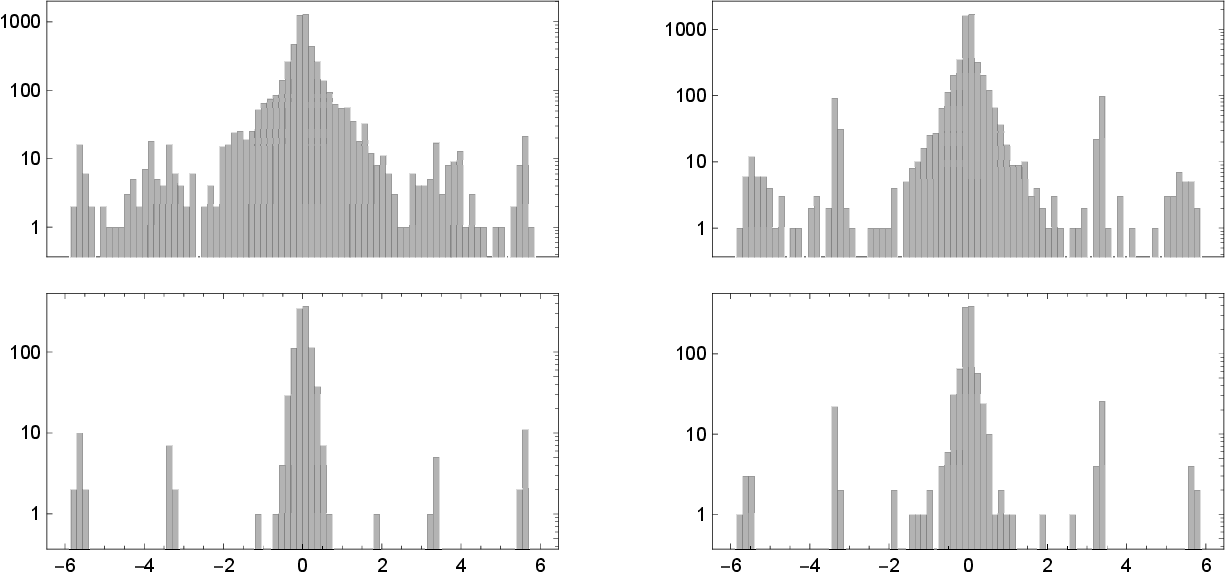}
    \caption{Distribution of  network weights after training on the two types of random logic data (left: \textsc{And/Or}, right: \textsc{Maj}). The lower histograms show the weights in just the first layer of each network.}
    \label{fig:wgthist}
\end{figure}

The diversification of the weights has an innocent explanation. In the ground truth models (Figure \ref{fig:randlogic}) there is a small probability that the values at two nodes in the same layer are perfectly correlated, or perfectly anticorrelated (when these are copied from the same node in the layer below). A gate taking input from one of these nodes could instead take input from the other node (negated, if anticorrelated) without changing its output. But it could also take a fixed mixture, with weights $p w$ and $(1-p)w$, and still produce the same output. This has the effect of decreasing the neuron's $w\cdot w$ norm, with the even mixture ($p=1/2$) giving the greatest decrease. 
To restore the norm, the constraint solver will scale up all of the neuron's weights, thereby also increasing the margin. The consolidation of nodes, as in this example of three relevant nodes becoming four, is inevitable when there is a loss of entropy in the higher layers of the network. It allows for a diversification of the network weights without compromising functionality.

Learning a Boolean function from examples is known as \textit{Boolean function inference}. In machine learning the same task would be described as generalization, albeit with rather strong constraints on the model. Models that implement simple logic operations over some small number of layers might be a good fit for various kinds of symbolic data. This article is not about those potential applications, but about a technology that can train such models.

\section{Boolean threshold functions as constraints}\label{sec:BTF}

A Boolean threshold function, or BTF, is a function $f: \{-1,1\}^m\to \{-1,1\}$ that can be represented by a weight vector $w\in\mathbb{R}^m$ as
\begin{equation}
f(x)=\sgn(w\cdot x)\;.
\end{equation}
We will use $X_m=\{-1,1\}^m$ for the hypercube of possible inputs and restrict weights as ``admissible'' by the property
\begin{equation}
\forall x\in X_m :\; w\cdot x\ne 0\;.
\end{equation}
In the machine learning context we will refer to the values $\pm 1$ as ``Boolean'', where $-1$ corresponds to the more standard 0 or \textsc{False}. In networks, a weight vector's negative elements act as \textsc{Not} gates. 

BTFs have been the subject of much research, mostly in connection with the \textit{Chow parameter problem} \cite{o2008chow}. This is the problem of constructing a representation---an admissible $w$---when given the set $X_+=f^{-1}(1)$ for some BTF (more specifically, parameters derived from $X_+$ first introduced by Chow). Thankfully, this difficult chapter of BTF history is largely irrelevant for the margin and support characteristics of BTFs that this work relies on.

Since we are always interested in the representations, we use the notation $f_w$ for BTFs when a question involves the representation/weights $w$. The number of inputs is always $m$. The property of BTFs that we will exploit in networks calls on the following definitions.

\begin{definition}
The support of a weight vector $w$, or $\supp(w)$, is the number of ``relevant'' inputs of $f_w$. Input $p$ is relevant if and only if there exists some $x\in X_m$ such that
\begin{equation*}
f_w(x_1,\ldots,x_p,\ldots,x_m)\ne f_w(x_1,\ldots,-x_p,\ldots,x_m)\;.
\end{equation*}
\end{definition}

\begin{definition}
The \textit{margin} of a weight vector $w$ is given by
\begin{equation}
\mu(w)=\min_{x\in X_m}|w\cdot x|
\end{equation}
and is strictly positive when $w$ is admissible.
\end{definition}
\noindent A BTF with $\supp(w)=1$ just acts as a wire that conveys a single input value to the output, with or without negation. There are no BTFs with $\supp(w)=2$. That's because if $w=(w_1,w_2)$, then to be admissible $|w_1|\ne |w_2|$. But the output of such a BTF only depends on the input having the largest magnitude weight, so $\supp(w)=1$. BTFs first become interesting at $\supp(w)=3$. 

Fully supported weight vectors have the following bound on their margin:
\begin{lemma}\label{lem:marginbound}
If $w$ is admissible and $\supp(w)=m$, then
\begin{equation}\label{lem:mubound}
\mu(w)\le w_\mathrm{min}=\min_p |w_p|\;.
\end{equation}
\end{lemma}
\begin{proof}
Let $p$ be an input for which $|w_p|$ is a minimum and call this number $w_\mathrm{min}$, where $w_\mathrm{min}>0$ since otherwise $\supp(w)<m$. Now consider the set
\begin{equation}
Y_{-p}=\big\{w\cdot x-w_p x_p: x\in X_m\big\}\;,
\end{equation}
that is, the set of threshold function values without the contribution from input $p$. At least one element $y^*\in Y_{-p}$ satisfies $|y^*|<w_\mathrm{min}$, since otherwise there are no elements $y\in Y_{-p}$ for which $y\pm w_\mathrm{min}$ have opposite signs and input $p$ would be irrelevant (so $w$ would not have full support). But $|y^*\pm w_\mathrm{min}|$ are bounds on the margin, so we find
\begin{equation}
\mu(w)\le \Big||y^*|- w_\mathrm{min}\Big|\le w_\mathrm{min}\;.
\end{equation}
\end{proof}

Because the weight vectors in our networks have a fixed normalization, the result we rely on the most is the following:
\begin{theorem}\label{thm:normmarginbound}
Let $w$ be an admissible weight vector with $\supp(w)=m$, then
\begin{equation}
\frac{\mu(w)}{\|w\|}\le\frac{1}{\sqrt{m}}
\end{equation}
and the equality case requires that $m$ is odd and all the elements of $w$ have equal magnitude.
\end{theorem}
\begin{proof}
Using lemma \ref{lem:marginbound} on the numerator,
\begin{align}
\frac{\mu(w)}{\|w\|}&\le\frac{w_\mathrm{min}}{\sqrt{\sum_{p=1}^m w_p^2}}\\
&\le \frac{w_\mathrm{min}}{\sqrt{m w_\mathrm{min}^2}}=\frac{1}{\sqrt{m}}\;.
\end{align}
Equality requires that all $m$ terms in the denominator are equal, or that the weight vector has equal-magnitude elements. But such a weight vector for even $m$ is not admissible, so the equality case is only possible for odd $m$.
\end{proof}

When weight vectors $w$ and $w'$ represent the same BTF we write $w\sim w'$. 
In our networks the number of BTF inputs $m$ may vary, and the weight vectors are normalized as
\begin{equation}\label{eq:weightnormalization}
\|w\|^2=m\;.
\end{equation}
With this choice, the rms value of the weight vector elements matches that of the Boolean inputs $x$ and the Boolean output $y$. Under the normalization constraint, maximizing the margin translates to fewer nonzero elements:
\begin{lemma}\label{lem:optweights}
Suppose an $m$-input BTF has $\supp(w)=n\le m$. Then there exists a $w^*\sim w$ with exactly $n$ nonzero weights and $\mu(w^*)\ge\mu(w)$.
\end{lemma}
\begin{proof}
The weights of $w$ associated with the $m-n$ irrelevant inputs can be set to zero without changing the BTF. To restore normalization \eqref{eq:weightnormalization} to the resulting $w^*$, the nonzero elements are rescaled by a factor $\lambda\ge 1$. This changes the margin by the same factor.
\end{proof}
\noindent This property of BTF representations combined with theorem \ref{thm:normmarginbound} will let us learn network representations of Boolean data that are sparse.

\subsection{Constraints on margins in training}\label{sec:marginsintraining}

During training each $m$-input BTF with weight vector $w$ is subject to the following constraint:
\begin{equation}\label{eq:margincon}
\forall x\in X_t:\quad \mu_m\le |w\cdot x|\;.
\end{equation}
Here $\mu_m$ is an $m$-dependent margin lower bound set by the user, and $X_t\subset X_m$ is the set of inputs seen by the BTF during training. $X_t$ might be a proper subset because the size of the training data is small, the node consolidation effect, or a combination of these. Because of this possibility it may be possible to satisfy all the margin constraints \eqref{eq:margincon} even when $\mu_m>\mu(w)$.

The possibility that $X_t$ is significantly smaller than $X_m$ is made unlikely by a combination of two things. First, in constraint-based training one is not forced to use small data batches as in the stochastic approach to gradient-based training. The second reason is slightly circular. When sparse network representations exist, where the BTFs all have small support, then it is only the size of the subset of the relevant inputs that matter. A BTF with support $n$ much smaller than $m$ only has to see the much smaller complete set of $2^n$ vectors on the relevant inputs in order to satisfy the more generous constraint
\begin{equation}\label{eq:marginlowerbound}
\mu_m\le \mu(w^*)
\end{equation}
promised by lemma \ref{lem:optweights} for the representation $w^*$ that has only $n$ nonzero elements.

For the reasons just given, it makes sense to define the $m$-dependent margin bound as
\begin{equation}
\mu_m=\sqrt{m/\sigma}\;,
\end{equation}
where $\sigma$ is the \textit{support hyperparameter}. Assuming a sparse network representation exists, with sparse BTF representations $w^*$, then applying theorem \ref{thm:normmarginbound} to a BTF with $n\le m$ nonzero elements in its weight vector $w^*$, we find
\begin{equation}
\mu(w^*)\le \frac{\|w\|}{\sqrt{n}}=\sqrt{m/n}\;.
\end{equation}
Theorem \ref{thm:normmarginbound} only applies when the inputs to the BTF are sufficiently diverse as described above, but if that is true, then \eqref{eq:marginlowerbound} is also true and we arrive at
\begin{equation}\label{eq:supportbound}
n\le \sigma\;.
\end{equation}
We now have a strategy for learning sparse network representations: force the BTF margins to be large with a small $\sigma$.

In our applications of the new training method we will encounter situations where the bound \eqref{eq:supportbound} is rigorous. When the network has $m$ inputs and data for all $2^m$ unique inputs to the first layer of BTFs is available, then in any solution to the constraint satisfaction problem those BTFs will have supports $n$ that are upper bounded by the parameter $\sigma$. In the event that the smallest $\sigma$ for which solutions exist is an odd integer $n^*$, then theorem \ref{thm:normmarginbound} also tells us that for the BTFs (in the first layer) with exactly $n^*$ nonzero weights, these will have equal magnitudes.

\subsection{Representing logic gates}\label{sec:replogicgate}

The weights $w$ of BTFs have simple structures when their margin $\mu(w)$ is maximized subject to a fixed norm constraint. The maximizing $w^*$ is uniquely determined by no more than $m$ ``extreme'' inputs $x$ of $f_w^{-1}(1)$ that have a common value of $w^*\cdot x$, and this value equals $\mu(w^*)$. Not surprisingly, the margin-maximizing $w^*$ are no different from the linear programming optimized weights that were tabulated as early as 1961 \cite{muroga1970enumeration}. Up to support $m=8$ these are integer vectors when normalized to have $\mu(w)=1$. For $m=9$ there are counterexamples, such as
\begin{equation}\label{eq:1171/2}
w={\textstyle \frac{1}{2}}(29,25,19,15,12,8,8,3,3)\;.
\end{equation}
To admit a BTF with this $w$, our ``support'' hyperparameter would have to satisfy $\sigma\ge \|w\|^2=1171/2$.

Packing highly complex decision encodings into individual neurons was not our motivation for using BTFs. Not only would the small margin constraints $\mu_m=\sqrt{m/\sigma}$ be problematic for robust learning, they would admit BTFs whose actual supports are very large. By keeping $\sigma$ small, the complexity of the representation is transferred from the processing internal to neurons, to the complexity of the wiring between them.

To add perspective to the objective just described, and defend our choice of the term ``support'', consider the case $\sigma=9$. The type of a BTF can be uniquely specified by a nonincreasing tuple of positive weights as in \eqref{eq:1171/2}. The full set of BTFs of a given type is obtained by applying all possible sign changes and permutations to the positive and ordered tuple. Here is the complete inventory of BTF types that are admitted with $\sigma=9$ (squared-norm $9$ or less) :
\begin{align*}
&(1,0,\;\ldots\;)\\
&(1,1,1,0,\;\ldots\;)\\
&(1,1,1,1,1,0,\;\ldots\;)\\
&(2,1,1,1,0,\;\ldots\;)\\
&(1,1,1,1,1,1,1,0,\;\ldots\;)\\
&(1,1,1,1,1,1,1,1,1,0,\;\ldots\;)\;.
\end{align*}
In all cases except one, the value of $\sigma$ translates directly to the support of the BTF, and the BTF implements a simple majority gate. We should not be dismissive of the simplicity of majority gates. After all, the 9-input gate can take 512 forms depending on the negation pattern of its weights.

In most of our applications we will set $\sigma=3$. This is sufficient to implement arbitrary Boolean functions in a deep network because the 3-input majority gate can implement 2-input \textsc{And/Or}:
\begin{align*}
\mbox{\textsc{Maj}}(-1,x_1,x_2)&=\mbox{\textsc{And}}(x_1,x_2)\\
\mbox{\textsc{Maj}}(+1,x_1,x_2)&=\mbox{\textsc{Or}}(x_1,x_2)\;.
\end{align*}
Our definition of BTFs with threshold value zero makes them \textit{self-dual}. Self-dual Boolean functions $f$ have the property $f(x)=-f(-x)$ for all inputs $x$. \textsc{And/Or} (for any number of inputs) are not self-dual because they are expressed in terms of a self-dual function with one input held constant. The BTFs in our networks all take input from a special node 0 with value $x_0=-1$. By choice of the corresponding weight $w_0$, a BTF can choose to be self-dual ($w_0=0$), and-like ($w_0>0$), or or-like ($w_0<0$). The special weight $w_0$ plays the same role as the bias parameter in standard neural networks.

Keeping a single input constant does not change any of the results in this section that used a margin defined in terms of a minimum over the entire hypercube of inputs. That's because
\begin{align*}
\min_{x\in X_m} |w\cdot x|&=\min_{x\in X^-_m\cup X^+_m} |w\cdot x|\\
&=\min\Big(\min_{x\in X^-_m} |w\cdot x|\;,\;\min_{x\in X^+_m} |w\cdot x|\Big)\\
&=\min\Big(\min_{x\in X^-_m} |w\cdot x|\;,\;\min_{x\in X^-_m} |-w\cdot x|\Big)\\
&=\min_{x\in X^-_m} |w\cdot x|\;,
\end{align*}
where $X^\pm_m$ is the half-hypercube whose points have $\pm 1$ for their first coordinate.

\subsection{Projecting to the BTF constraint}\label{sec:BTFproj}

The set $A$ corresponding to the BTF constraint is the union $A= A_+\cup A_-$, where (all upper or all lower signs)
\begin{equation*}
\begin{aligned}
A_\pm
&=\Big\{(w,x,y)\in\mathbb{R}^{m+m+1}:~
\pm w\cdot x\ge \mu_m,\; y=\pm 1\Big\}\;.
\end{aligned}
\end{equation*}
Projecting to the BTF constraint means, for any given $(w,x,y)\in\mathbb{R}^{m+m+1}$, compute the point $(w',x',y')\in A$ that minimizes the distance \eqref{eq:BTFdistance}. The projection is unique except for $(w,x,y)$ on a set of measure zero.

There are four cases, depending on whether $w\cdot x$ and $y$ have the same sign, and whether $|w\cdot x|$ exceeds the margin. For the easiest case, where the signs match and $|w\cdot x|\ge \mu_m$, the projection only needs to change $y$ : $y'=\sgn(y)$. In the most compute-intensive case, where both of these conditions are violated, two outcomes must be considered:
\begin{align*}
w'\cdot x'=+\mu_m,&\qquad y'=+1\\
w'\cdot x'=-\mu_m,&\qquad y'=-1\;.
\end{align*}
The projection of $(w,x)$ to these bilinear constraints can be worked out using the method of Lagrange multipliers. The result,
\begin{align*}
w'&=\frac{w+\lambda x}{1-\lambda^2}\\
x'&=\frac{x+\lambda w}{1-\lambda^2}\;,
\end{align*}
is expressed in terms of the unique zero $\lambda\in (-1,1)$ of the function
\begin{equation*}
h_{\pm}(\lambda)=\frac{(1+\lambda^2)w\cdot x+\lambda (w\cdot w+x\cdot x)}{(1-\lambda^2)^2}\mp \mu_m\;,
\end{equation*}
where the upper sign corresponds to the case $y'=+1$. In large networks the work to compute this projection scales as $m$, the number of vector elements that are updated and terms in the three inner product computations. An efficient root-finding method is described in \cite{elser2021learning, bauschke2022projecting}. 

\subsection{Symmetries}\label{sec:sym}

In addition to any permutation symmetries of the network, BTF networks have another symmetry not shared by the networks of standard machine learning. This is a local symmetry at each of the network's ``hidden'' nodes, defined to be nodes that both receive and send Boolean values to other nodes. This symmetry is a generalization of De Morgan's law and follows directly from the self-dual property of the BTF:
\begin{align*}
f_w(x)&=-f_w(-x)\\
&=-f_{-w}(x)\;.
\end{align*}
Simultaneously flipping the signs of the weights on all the incoming and outgoing edges of a hidden node has no effect on the Boolean function the network is expressing.

\section{Divide-and-concur for BTF networks}\label{sec:DC}

In this section we consider general feed-forward network architectures. There is no reference to layers, or the special node that holds the constant input. $\mathcal{N}$ is the set of all nodes; hidden, input and output nodes are denoted $\mathcal{H}$, $\mathcal{I}$, and $\mathcal{O}$. The network edges are denoted $\mathcal{E}$ and $\mathcal{D}$ is the set of data labels.

Variables have both a network and data label. At each node $p$, and for each data item $i$, there is an ``output variable'' $y_{p\, i}$. These hold network inputs when $p\in \mathcal{I}$ and BTF outputs otherwise. The other variables live on network edges, labeled $p\to q$ when the BTF at node $q$ receives input from node $p$. The ``input variables'' $x_{p\to q\, i}$ have this structure to allow them to satisfy the constraints for BTF $q$ without having to be consistent with $y_{p\, i}$. The ``weight variables''  $w_{p\to q\, i}$ clearly also have an edge label, and $i$ labels the data-item copy that allows the BTF constraints to be satisfied independently for each $i$.

The groupings of the three variable types in the divide-and-concur framework are shown in Figure \ref{fig:DCnet}. Highlighting in the left panel shows ``divide'': a disjoint union by BTFs, each one a fan-in of inputs (red lines), weights (green lines), with a single BTF output (blue disk). The constraints on each group are independent and together comprise the $A$ constraint. The highlighted fan-out in the right panel shows a different disjoint union. The constraint here is the concurrence of one BTF output with all the inputs that receive that output. Taken together, the independently concurring groups are part of the $B$ constraint. The $B$ constraints also includes the concurrence of weights across data items (green lines across multiple network instantiations).

\begin{figure}
    \centering
    \includegraphics[width=\columnwidth]{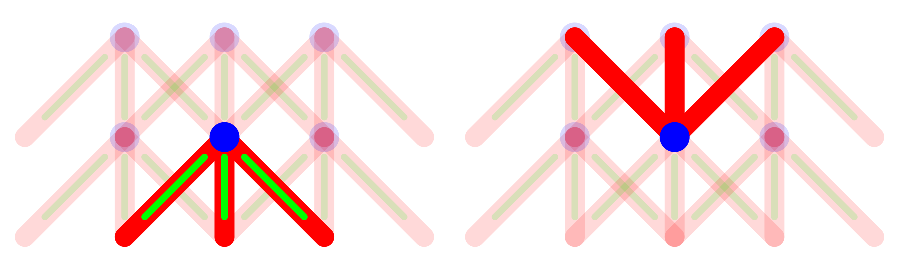}
    \caption{Divide-and-concur variable groupings in BTF networks. Red lines are neuron inputs ($x$), blue disks are neuron outputs ($y$), and green lines are weights ($w$). On the left, the highlighted fan-in shows all the variables in one BTF constraint. On the right, the highlighted fan-out show the concur between neuron outputs and inputs.}
    \label{fig:DCnet}
\end{figure}

More formally, $A$ has the following product structure
\begin{equation*}
A={\textstyle \prod}_{q\,\in\, \mathcal{H}\cup \mathcal{O}}^\times\;\;{\textstyle \prod}_{i\,\in\, \mathcal{D}}^\times \;\;A_{q\, i}\;.
\end{equation*}
Each factor has the form given in Section \ref{sec:BTFproj} with $y$ replaced by $y_{q\, i}$, $w$ and $x$ by the vectors with elements $w_{p\to q\, i}$ and $x_{p\to q\, i}$, where $p$ ranges over all nodes incident on $q$, which we denote $p\in \;\to\!\!q$. The $A$ constraint also has a factor associated with the input nodes. This is described later, as it involves the data.

The output $y_{p\, i}$ of BTF $p$ for data item $i$ and the inputs $x_{p\to q\, i}$ sent to BTFs $q$ (for the same data item) are no longer independent in the ``concur'' or $B$ constraint. Even so, $B$ also has a product structure:
\begin{equation*}
B={\textstyle \prod}_{p\,\in\, \mathcal{I}\cup \mathcal{H}}^\times\;\;{\textstyle \prod}_{i\,\in\, \mathcal{D}}^\times \;\;B_{p\, i}\;.
\end{equation*}
Each $B_{p\, i}$ is an independent constraint defined by the system of equalities
\begin{equation}\label{eq:xyconcur}
\forall q\in\; p\! \to\;\;  :\quad x_{p\to q\, i}=y_{p\, i}\;,
\end{equation}
where $p\! \to$ denotes the set of nodes on which $p$ is incident. The $B$ constraint also has a factor associated with the data (described later), now at all the output nodes.

To keep the weights across all the data items consistent, the $B$ constraint also has the following factors, one at each BTF:
\begin{equation*}
{\textstyle \prod}_{q\,\in\, \mathcal{H}\cup \mathcal{O}}^\times\;\; B_q\;.
\end{equation*}
Each factor labeled $q$ combines weight-equality across data items and weight normalization:
\begin{equation}
\forall p\in\; \to\! q\;  :\quad w_{p\to q\, i}=\bar{w}_{p\to q}\;,\label{eq:wconcur}
\end{equation}
\begin{equation}
{\textstyle \sum}_{p\,\in\; \to q} \;\bar{w}_{p\to q}^2=| \!\to\! q|\;.\label{eq:wside}
\end{equation}
The first statement simply states that every data-item  copy $i$ is equal to the same ``concur'' value, with symbol $\bar{w}_{p\to q}$. The squared-norm equals the number of inputs to the BTF, $| \!\to\! q|$ (denoted $m$ in Section \ref{sec:BTF}).

The $B$ constraint is so simple that with essentially no extra work we can compute the projection to this set for a natural generalization of the metric:
\begin{equation}\label{eq:metric}
\|z'-z\|_g^2=\sum_{q\in \mathcal{N},\;i\in\mathcal{D}}g_{q\, i}\,\|z'_{q\, i}-z_{q\, i}\|^2\;,
\end{equation}
where the symbol $z$ is being used for the aggregate of $(w,x,y)$ variables,
\begin{equation*}
\|z'_{q\, i}-z_{q\, i}\|^2=(y'_{q\, i}-y_{q\, i})^2
\end{equation*}
for $q\in \mathcal{I}$, and otherwise
\begin{equation*}
\begin{aligned}
\|z'_{q\, i}-z_{q\, i}\|^2
&={\textstyle \sum}_{p\,\in\, \to q}
\Big(
  (w'_{p\to q\, i}-w_{p\to q\, i})^2 \\
&\qquad\qquad
 + (x'_{p\to q\, i}-x_{p\to q\, i})^2
\Big) \\
&\qquad\qquad + (y'_{q\, i}-y_{q\, i})^2 \;.
\end{aligned}
\end{equation*}
is the same as expression \eqref{eq:BTFdistance} of Section \ref{sec:overview} (so that the projection to the BTF constraint is unchanged). The positive numbers $g_{q\, i}$ are the \textit{metric coefficients}. For the input nodes $q$, and all data items $i$, we set $g_{q\, i}=1$. However, we will see that allowing the metric coefficients to be nonuniform elsewhere enhances the performance of the constraint satisfaction algorithm. A simple, automated heuristic for tuning the metric is described in Section \ref{sec:RRR}. In this section we treat the metric coefficients as fixed, given numbers.

Derivations of the concur values for \eqref{eq:xyconcur}, in the case of a nonuniform metric, can be found in \cite{elser2025solving} :
\begin{equation*}
\bar{y}_{p\, i}=\frac{{\textstyle \sum}_{q\,\in\,p\to}\;g_{q\,i}\,x_{p\to q\, i}\; + \;g_{p\,i}\,y_{p\, i}}{{\textstyle \sum}_{q\in\,p\to}\;g_{q\,i}\; +\; g_{p\,i}}\;.
\end{equation*}
For the projection of the weights in the $B$ constraint one first computes
\begin{equation*}
\bar{w}_{p\to q}=\frac{{\textstyle \sum}_{i\in\mathcal{D}}\;g_{q\, i}\;w_{p\to q\, i}}{{\textstyle \sum}_{i\in\mathcal{D}}\;g_{q\, i}}\;,
\end{equation*}
and rescales this average by the factor
\begin{equation*}
\sqrt{\frac{|\!\to\! q|}{{\textstyle \sum}_{p\,\in\,\to q}\;\bar{w}_{p\to q}^2}}
\end{equation*}
to satisfy the side constraint \eqref{eq:wside}.

\subsection{Input and output constraints}

The data value imposed at network input $p\in\mathcal{I}$, for data item $i$, is written $y^A_{p\,i}$ because this constant takes the place of the projection to the BTF constraint---called $A$ above---at the input nodes (where there is no BTF). Similarly, the data value imposed at network output $q\in\mathcal{O}$, for data item $i$, is written $y^B_{q\,i}$ because this constant takes the place of the projection to the concur constraint---called $B$ above---at the output nodes (from which there are no $x$ variables to concur with). The corresponding data factors of $A$ and $B$ are simply
\begin{subequations}\label{eq:datacon}
\begin{align}
A_\mathcal{I}&=\Big\{y\in\mathbb{R}^{|\mathcal{I}|\times |\mathcal{D}|}\;:\quad y_{p\,i}=y^A_{p\,i}\Big\}\\
B_\mathcal{O}&=\Big\{y\in\mathbb{R}^{|\mathcal{O}|\times |\mathcal{D}|}\;:\quad y_{q\,i}=y^B_{q\,i}\Big\}\;.
\end{align}
\end{subequations}

In \texttt{boolearn} the data values in files are always 0 or 1, and these are converted to $-1$ and $+1$, respectively, to be the $y^A$ values used by the constraint satisfaction algorithm. 
There is one exception, when the input data is analog and sufficiently close to Boolean, like the pixel values of MNIST. In this case, \texttt{boolearn} assumes the input data lies in the interval $[0,1]$ and these are mapped linearly into the interval $[-1,1]$ to give $y^A$.

The training program in \texttt{boolearn} can be run in a generative mode, where no input data is given. Two options for constraining the inputs in that case are available. The weakest is the constraint
\begin{equation}\label{eq:inputbool}
A_\mathcal{I}=\Big\{y\in\mathbb{R}^{|\mathcal{I}|\times |\mathcal{D}|}\;:\quad y_{p\,i}\in\{-1,1\}\Big\}\;,
\end{equation}
which imposes Boolean values when there is no BTF to do that. Projecting to this constraint means rounding to $-1$ or $1$, independently for each input $p$ and data item $i$. The other option is the 1-hot constraint:
\begin{equation}\label{eq:input1hot}
A_\mathcal{I}=\Big\{y\;:~ y_{p\,i}\in\{-1,1\}\;,\; {\textstyle \sum}_{p\in\mathcal{I}}\; y_{p\, i}=2-|\mathcal{I}|,\Big\}\;
\end{equation}
where $y\in\mathbb{R}^{|\mathcal{I}|\times |\mathcal{D}|}$. To project to this constraint, independently for each $i$, the largest input $p$ is singled out and replaced by $1$, all others by $-1$.

\section{Constraint satisfaction with RRR}\label{sec:RRR}

By using the divide-and-concur strategy, in Section \ref{sec:DC} we showed how the BTF-network training problem can be cast in the form
\begin{equation}\label{eq:AcapB}
\mathrm{find\; some}\;z\in A\cap B\;,
\end{equation}
where $z$ is a point in the space of the aggregated $(w,x,y)$ variables, $A$ is the set in that space that defines the operation of BTFs, and $B$ defines the network's architecture. The training data provides additional constraints, and these are included with $A$ at the inputs and $B$ at the outputs. In this section we describe an algorithm for solving \eqref{eq:AcapB} called reflect-reflect-relax (RRR), a generalization of Fienup's hybrid-input-output phase retrieval algorithm. RRR makes no reference to the ``divide'' character of $A$ and the ``concur'' character of $B$. In fact, there are many applications of RRR where $A$ and $B$ do not have those interpretations at all.

In order to use RRR, the one property required of $A$ and $B$ is that the two projections,
\begin{align*}
P_A(z)&=\argmin_{z'\in A}\;\|z'-z\|\\
P_B(z)&=\argmin_{z'\in B}\;\|z'-z\|\;,
\end{align*}
can be computed efficiently. What we showed in Section \ref{sec:DC} is that not only can these be computed efficiently for the BTF-network training problem, but that the computations distribute, thanks to the product structures of $A$ and $B$. In the case of $P_A$, the variables assigned to each BTF in the network and each item in the data are independently projected to satisfy the $\sgn$ ``activation function'' with its lower bound on the margin. $P_B$ is distributed over all the groups of BTF outputs $y_{p\,i}$ and receiving-BTF inputs $x_{p\to q\,i}$, and also independently over the data $i$. $P_B$ also concurs the weights (across data), independently for each edge of the network.

$P_A$ and $P_B$ are computed once in each iteration of RRR, and the work is proportional both to the number of network edges and the number of data items. The work per RRR iteration is therefore comparable to the work performed in gradient descent over a single pass through the data.

For the history of RRR and the logic behind the iterations rule,
\begin{equation}\label{eq:RRR}
z\mapsto z'=z+\beta\Big(P_B(R_A(z))-P_A(z)\Big)\;,
\end{equation}
we recommend \cite{elser2025solving}. In this section we only touch on the minimum needed to use the algorithm.

It is important that the argument of $P_B$ in \eqref{eq:RRR} is not $z$ but the reflection
\begin{equation}
R_A(z)=2P_A(z)-z\;.
\end{equation}
At a fixed point $z^*\mapsto z^*$ we have
\begin{equation}
P_B(R_A(z^*))-P_A(z^*)=0
\end{equation}
which implies
\begin{equation}
P_B(R_A(z^*))=P_A(z^*)=z_\mathrm{sol}
\end{equation}
is a solution of \eqref{eq:AcapB} since it lies in the range of both projections. While this conclusion did not rely on using $R_A$, without the reflection the iterate $z$ does not converge to fixed points.

The RRR iteration rule ensures there is a kind of \textit{local convergence}, independent of whether the problem is feasible. As we will see, this is the only thing the user needs to know. Concerns about global convergence are misplaced since even truly hard problems (NP-complete) have formulations
(31) with easy projections. A proof of convergence (in polynomially-bounded iterations) in the general case would mean that even these hard problems have easy solutions.

\begin{figure}[t]
\centering
\includegraphics[width=\columnwidth]{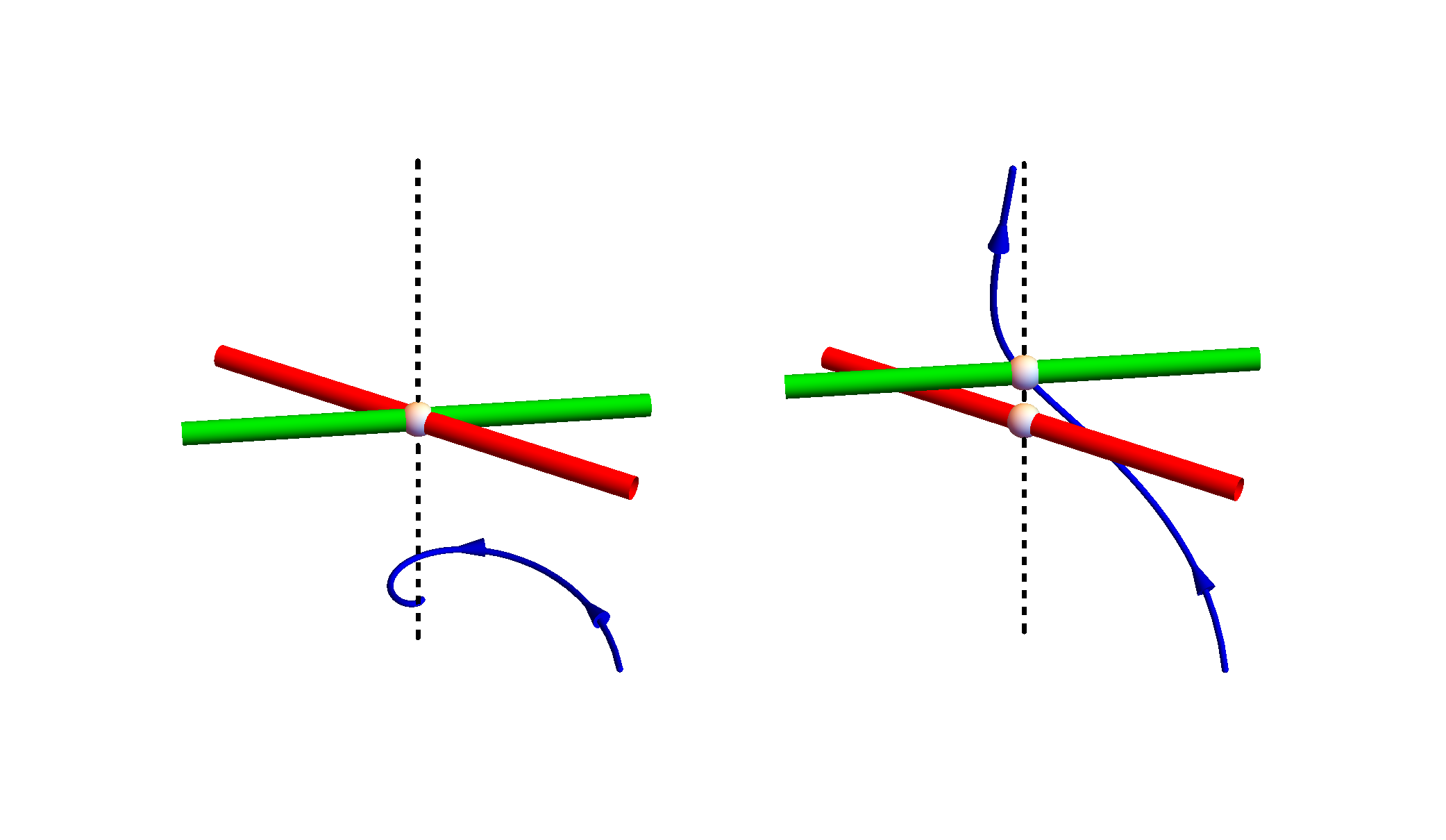}
\caption{Local behavior of the iterate $z$ (blue) in the limit of small $\beta$ when the problem is locally feasible (left) and locally infeasible (right). $A$ and $B$ can locally be approximated as affine and are  rendered as the red ($A$) and green ($B$) rods in this cartoon.}
\label{fig:convergence}
\end{figure}

A cartoon of the local behavior is sketched in Figure \ref{fig:convergence}. The trajectory of $z$ is shown for the limit of small $\beta$, where it is a smooth curve, but the convergent behavior holds for any $\beta\in (0,2)$. In both the feasible and infeasible cases there is a pair of \textit{proximal points} on the two sets. These coincide when the problem is feasible. In either case, there is an affine space orthogonal to $A$ and $B$ (dashed line) on which $z$ converges (locally). When the problem is feasible this space is the set of fixed points $z^*$ of RRR and for any of them
\begin{align*}
z_A&=P_A(z^*)\\
z_B&=P_B(R_A(z^*))
\end{align*}
are the coincident solution points. This many-to-one relationship of fixed points to solutions makes RRR different from more familiar fixed-point algorithms, like Newton's root finding method.

The local behavior in the infeasible case is the \textit{ars arcana} of RRR. The algorithm still generates the points $z_A$ and $z_B$, now from a $z$ that converges to the orthogonal space while also moving along this space in the direction $A$ (red) to $B$ (green). If $z$ is able to completely converge to the orthogonal space, then the corresponding $z_A$ and $z_B$ are a proximal pair. But if the local gap $\Delta=\|z_B-z_A\|$ is large, the speed of the iterate $\beta\Delta$ will also be large and in just a few iterations the picture of local constraints will have changed. If by luck that new local picture is feasible, then RRR converges to one of the solution's fixed points. Proximal infeasible points are the bane of the more obvious algorithm that alternates the two projections (and gets stuck on the proximal points).

The likelihood of infeasibility, globally, must be taken seriously when the data that defines $A$ and $B$ is noisy (phase retrieval) or the model is insufficiently expressive (machine learning). In this case the best that RRR can offer is a pair of proximal points whose gap $\Delta$ is exceptionally small, say relative to its starting value. Whether $z_A$ or $z_B$ should be used as the best-possible approximate solution depends on the application. 
The correct choice when training networks is $z_B$, since the weights concur over data items in set $B$.

\subsection{Metric auto-tuning}

As explained in Section \ref{sec:DC}, the two projections are just as easy to compute when the metric coefficients $g_{q\, i}$ in \eqref{eq:metric} depend on the BTF $q$ and the data item $i$. The training program in \texttt{boolearn} adjusts these in response to the values of the local gaps
\begin{equation}
\Delta_{q\, i}=\|z^A_{q\, i}-z^B_{q\, i}\|\;,
\end{equation}
where $z_{q\, i}$ are the trio of $(w,x,y)$ at BTF $q$ and data-instantiation $i$ of the network, and the superscripts denote their projections to $A$ and $B$. The rms value of these local gaps,
\begin{equation}
\Delta^2_\mathrm{rms}=\frac{1}{|\mathcal{H}\cup\mathcal{O}|\;|\mathcal{D}|}\sum_{q\,\in\, \mathcal{H}\cup\mathcal{O}}\;\sum_{i\,\in\, \mathcal{D}}\;\Delta_{q\, i}^2
\end{equation}
is the same, apart from a contribution from the input nodes, as what we have defined as the gap $\Delta$, and whose value is zero at a solution. When $\Delta_{q\, i}$ is significantly larger than $\Delta_\mathrm{rms}$, say when averaged over some number of iterations, it means the variables for BTF $q$ and data item $i$ are changing much more from iteration to iteration than the variables for other BTFs and data items. This situation will benefit from an intervention that redistributes the variable changes more evenly, since the solution process in hard problems usually involves some cooperativity. To make variables with a large $\Delta_{q\, i}$ change less than their peers, we preferentially increase their metric coefficients by the rule
\begin{equation}
g_{q\, i}\mapsto g_{q\, i}+\gamma\Big(\Delta^2_{q\, i}/\Delta^2_\mathrm{rms}-g_{q\, i}\Big)\;,
\end{equation}
where $\gamma>0$ is the \textit{metric relaxation hyperparameter}. This update is applied in every iteration and preserves the average of the metric coefficients. Initially all metric coefficients are set to the value 1, including those for $q\in\mathcal{I}$ (which are not included in the update rule). It's important to keep $\gamma$ small so as not to upset the local convergence of RRR (which holds rigorously only for a constant metric). A good rule of thumb is to set $\gamma$ near the inverse of the total number of iterations to be performed, so the net change in the metric coefficients can be of order 1. 

\section{Applications}

\subsection{Multiplier circuits}

Our first application was chosen to make it easy to check the validity of the results. The data is the complete multiplication table, in base 2, for 2-bit integers (from $0\times 0=0$ to $3\times 3=9$). The network is tasked with learning how the $2+2$ bits of the factors are transformed into the 4 bits of the product, by BTFs acting over some number of layers. For maximum interpretability we use $\sigma=3$.

\begin{figure}[t]
\centering
\includegraphics[width=\columnwidth]{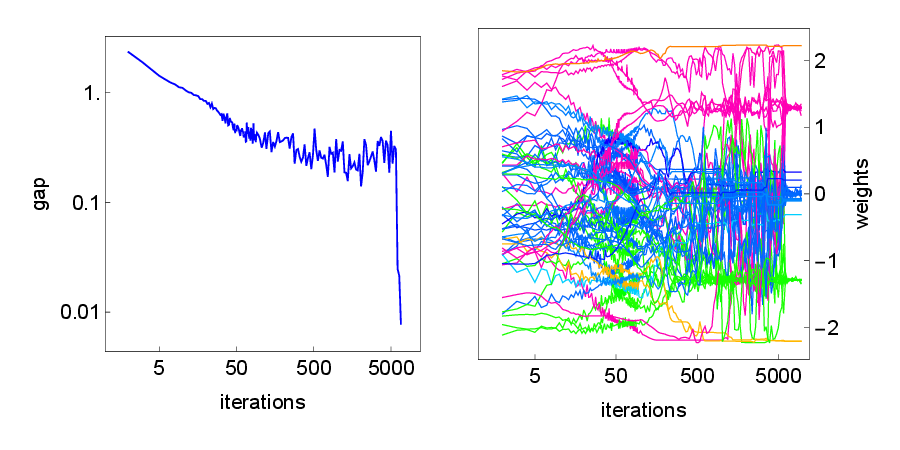}
\caption{Evolution of the gap (left) and weights (right) during training of a simple binary multiplier circuit. Though the evolution of the weights is chaotic, the values in the solution (when the gap is very small) are quite simple.}
\label{fig:3layermultplots}
\end{figure}

This problem is small enough that through experimentation we are able to trim the depth and width of (fully connected) networks to get a compact representation. The architecture $4\to 4\to 4\to 4$ appears to be the optimal 3-layer architecture because, empirically, the gap remains positive when the widths of either of the hidden layers is reduced to 3. Figure \ref{fig:3layermultplots} (left panel) shows the evolution of the gap in a run with hyperparameters $\beta=0.2$, $\gamma=10^{-3}$. These are good default values, even for much larger problems. The abrupt drop in the gap, enhanced by the logarithmic iterations axis, marks the transition from a period of search to the much faster convergent behavior to the space of fixed points. The corresponding evolution of the 60 weights in the network is shown in the right panel.

As explained at the end of Section \ref{sec:marginsintraining}, the equal-weight-magnitude conclusion of theorem \ref{thm:normmarginbound} may not hold when the BTFs do not see all patterns of Boolean values on their relevant inputs in training. Even so, we see that after just a few hundred iterations there is a concentration of weights with magnitude $\sqrt{5/3}\approx 1.29$, where $5$ is the number of inputs of our BTFs if we include the constant input that implements bias. The evidence suggests that weights are tending toward 3-input gates long before the solution is found. In the short, final convergent phase of the evolution, the conclusion of theorem \ref{thm:normmarginbound} is upheld perfectly.

Three of the weights converge to large magnitudes. If a BTF in our network had just a single nonzero weight, its magnitude would be $\sqrt{5}\approx 2.24$. The large weights in the solution can have magnitudes smaller than this without changing the 1-relevant-input property of the BTF (the weights to the non-relevant inputs will be small and are also in evidence in the plot).

The circuit found using \texttt{boolearn} is shown in Figure \ref{fig:3layermult}. Three of the BTFs (gray nodes) ended up with just one relevant weight and simply copy a BTF output from the layer below. All the others, with three relevant inputs, implement \textsc{And} (blue nodes) and \textsc{Or} (red node) because one of their relevant inputs is the constant bias node. For the circuit rendering, the signs of the weights incident to hidden nodes were flipped to minimize the number of \textsc{Not} gates (red arrows). The network's input and output nodes are labeled with the place-values of the bits in the factors and product. A few mathematical facts are brought to light by the circuit. That the 1's bit in the product is the \textsc{And} of the 1's bits in the factors is the statement that odd numbers always have odd factors. And from the structure of the first two layers we see that multiplication is indeed commutative.

The 2-bit multiplier problem is small enough that a circuit that minimizes the total number of 2-input gates (on any number of layers) can be found by exhaustive search. The result of that search (details in Supplemental Materials) is 8 gates, one less than the circuit found by \texttt{boolearn}.

\begin{figure}[t!]
\centering
\includegraphics[width=0.5\columnwidth]{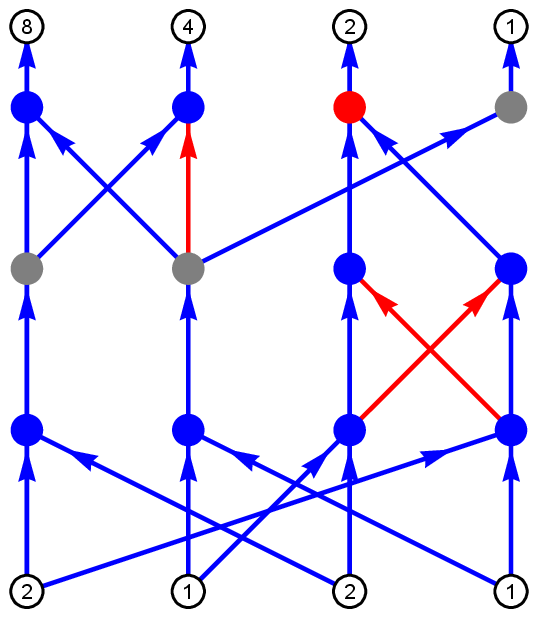}
\caption{A three-layer 2-bit multiplier circuit found using \texttt{boolearn}. There are eight \textsc{And} gates (blue nodes), one \textsc{Or} (red node), and three \textsc{Not}s (red arrows).}
\label{fig:3layermult}
\end{figure}

A smaller circuit, with architecture $4\to 5\to 4$, can be discovered by using $\sigma=5$. But admitting 5-input majority gates may increase the complexity of the representation. In any case, the exercise is instructive, now because of the three kinds of gates that can arise. In one solution found by \texttt{boolearn} there were three 5-input majority gates (all with bias), five 3-input majority gates (four with and one without bias), and only one BTF with a single relevant input.

Representations of the 3-bit multiplication table (from $0\times 0=0$ to $7\times 7=49$) should have more layers but also greater width, for storing intermediate results. Figure \ref{fig:3x3gap} compares the evolution of the gap for 3- and 4-layer architectures, both having 12 BTFs in each of their hidden layers. With the 4-layer architecture a solution is found in about $5\times 10^5$ iterations. In contrast, the gap for the 3-layer architecture appears to fall into a nonzero steady-state, casting doubt on the existence of a solution.

\begin{figure}[t!]
\centering
\includegraphics[width=.7\columnwidth]{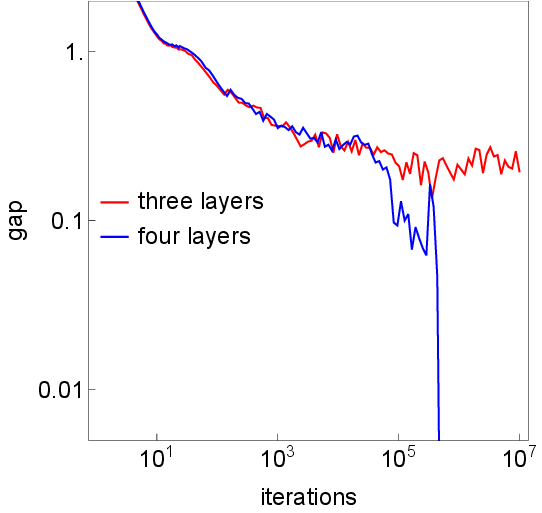}
\caption{The gap for learning the 3-bit multiplication table goes to zero for the $6\to 12\to 12\to 12\to 6$ architecture but remains positive when one layer of hidden nodes is removed.}
\label{fig:3x3gap}
\end{figure}

One would not use \texttt{boolearn} to design circuits for $k$-bit multipliers! We have no reason to believe the circuits found from data at particular $k$'s generalize to arbitrary $k$. Even so, this more limited exercise has demonstrated that \texttt{boolearn} generates interpretable representations and that the capacity for depth is not unique to the gradient approach. 

\subsection{Binary encoding-decoding}\label{sec:autoencoder}

One of the most convincing demonstrations of learning by error-back-propa\-gation in Rumelhart et al. \cite{rumelhart1985learning} was the ``binary'' autoencoder. The authors studied the (fully connected) architecture $2^k\to k\to 2^k$ and gave results for $k=3$. Learning the identity map from data comprising $2^k$ linearly independent vectors requires nonlinearities in the hidden layer, since otherwise the output can only have the linear span of the narrowest layer ($k$). Rumelhart et al. used 1-hot data vectors and achieved perfect accuracy with the sigmoid activation function. However, their results fell short of true binary encoding-decoding because on some of the data the code values (sigmoid outputs) included the number 1/2.

Unlike the gradient approach, where zero-loss only implies perfect reconstruction accuracy, when the constraint-based method
achieves a gap of zero we also know that the code is perfectly
binary. Zero-gap implies a feasible point has been found, which includes the property that the concur projections $y^B$ are equal to the Boolean $y^A$ projections. Using the same architecture as Rumelhart et al. and 1-hot data vectors (a single $+1$ replacing one element of a vector of all $-1$), the RRR algorithm easily learns weights for true binary encoding-decoding. We could present those results, not just for $k=3$ but also some larger $k$'s. Instead, we report on something more interesting.

There is no symmetry between the encoding and decoding stages of the autoencoder, but to see this asymmetry the data vectors have to be more generic than 1-hot vectors. We found that when the data vectors are drawn from the uniform distribution on the hypercube, the simple 2-layer architecture fails, even when the margin constraint is weakened with a large $\sigma$. For $k=3$ the gap fluctuates indefinitely around the value 0.2\,. Through experiment we discovered the problem is in the decoder. Augmenting just this stage, so the autoencoder has architecture $2^k\to k\to 2^k\to 2^k$, makes the constraint problem feasible.

Figure \ref{fig:autogap} shows the evolution of the gap and training accuracy with RRR iterations for generic instances of the autoencoder problem and $k=3,4,5$. The hyperparameters in these experiments were $\sigma=7$, $\beta=0.2$ and $\gamma=10^{-4}$. The training accuracy is the fraction of output bits that are correct (match the corresponding input bits). This reaches 100\% when the gap undergoes a sharp drop. That a gap of $0.01$ remains a reliable stopping criterion, even when networks are doubled in size, adds support to our normalization conventions.

\begin{figure}[t!]
\centering
\includegraphics[width=\columnwidth]{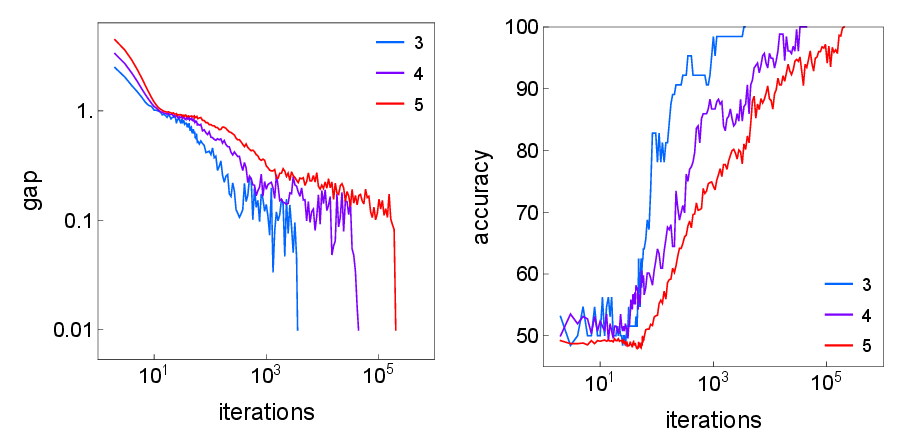}
\caption{Evolution of the gaps (left) and training accuracies (right) in the binary autoencoder problem for $k=3, 4, 5$.}
\label{fig:autogap}
\end{figure}

The role of the extra processing layer in the decoder is a mystery to us. Figure \ref{fig:decode1} compares our $k=5$ random data vectors with the vectors generated by the first stage ($k\to 2^k$) of the decoder. Clearly there is something special about the latter, since the 2-layer architecture ($2^k\to k\to 2^k$) would have worked with these as the data. We confirmed this and found that solutions were found on average in only 900 iterations. It would be interesting to know what property allows the transformed data to be decoded in just one layer.  

\begin{figure}[t!]
\centering
\includegraphics[width=\columnwidth]{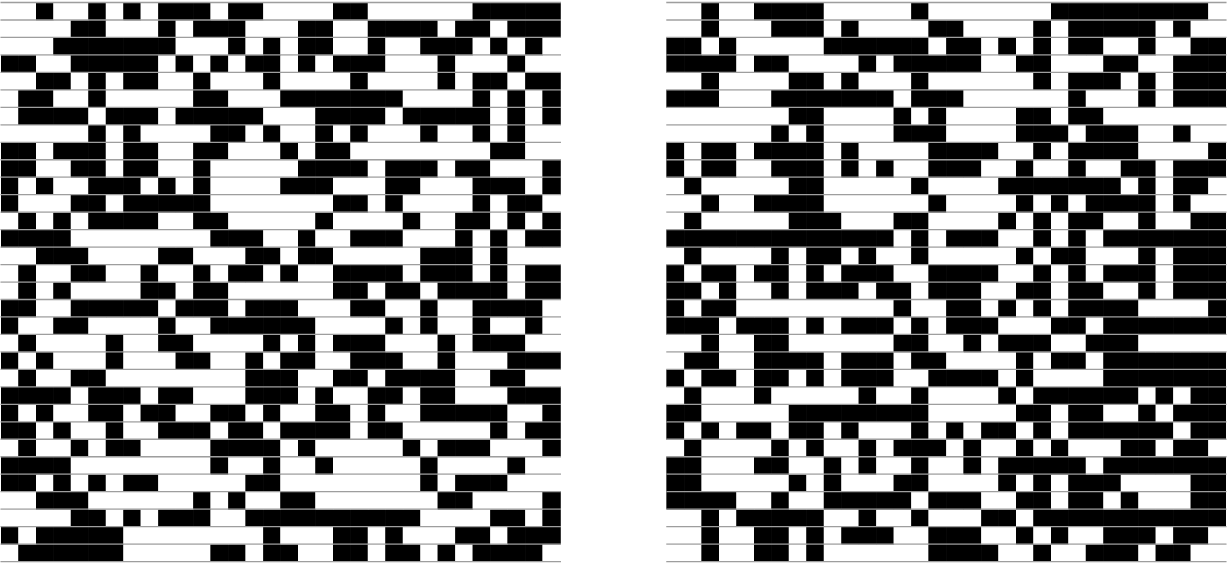}
\caption{Generic (left) and transformed (right) data vectors (rows) for the $k=5$ binary autoencoder problem.}
\label{fig:decode1}
\end{figure}

A decoding network can be learned directly in \texttt{boolearn} by using the ``any Boolean code'' option---constraint \eqref{eq:inputbool}---for the inputs. The network now has architecture $k\to 2^k\to 2^k$ and instead of imposing data values at the input nodes in the $A$ projection, the $y$ variables are simply rounded to $-1$ or $+1$, whichever is closer. \texttt{boolearn} keeps track of which input codes are getting used, and with what frequency. Since we are training on $2^k$ random and distinct data vectors in the output, each of the $2^k$ binary codes must be used exactly once in a solution.

Keeping $\beta=0.2$ and $\gamma=10^{-4}$, we were surprised to find that training the decoder, for $k=5$, was much easier with a \textit{smaller} value of $\sigma$. Limiting each solution attempt at $2\times 10^6$ iterations, the success rate was 100\% in 20 attempts for $\sigma=5$ (average iteration count $8.6\times 10^5$), and 0\% for $\sigma=7$. The success rate at $\sigma=3$ was also 0\%, but that is explained if the large BTF margin makes the constraint problem infeasible. On the other hand, any solution found with $\sigma=5$ would still be feasible with the smaller margin-bound setting of $\sigma=7$. This finding indicates that tightening the constraints and thereby decreasing the size of the feasible set can make it \textit{easier} for the RRR algorithm to find a point in that set.

\subsection{Generating MNIST 4's}

Data taken from the wild usually does not come with the guarantee it can be generated by a network of BTFs. To show how \texttt{boolearn} can be used in that setting, we turn to MNIST, a repository of representations of the digits 0--9 used by early humans. We selected just the 4's, cropped the images to $16\times 16$, and binarized them. The last step is straightforward since the pixel contrast distribution is strongly bimodal. To generate 1024 of the resulting binary strings (each of length $16^2$) with a binary code, as in the last application, would require a 10-bit code. But there is no reason at all to use the hypercube code for this data, with its large range of distances between codewords. We will instead use the 1-hot code and input-constraint \eqref{eq:input1hot}. The 1-hot codewords are equidistant, and if we want to try to generate all the data with $k$ codewords, our network will have $k$ inputs. Since the data vectors are distinct, our representation of MNIST 4's will be approximate in some sense because the constraint problem is infeasible. The RRR algorithm finds proximal point-pairs when a problem is infeasible and it will be interesting to see what that means in this application.

\begin{figure}[t!]
\centering
\includegraphics[width=.6\columnwidth]{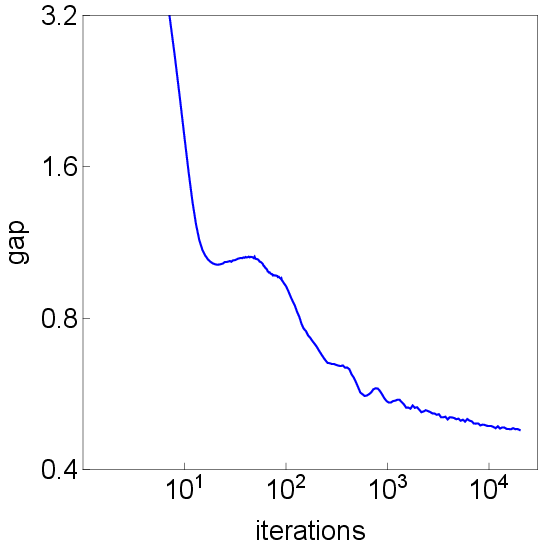}
\caption{The gap when training a generator of MNIST 4's from just 16 1-hot codewords.}
\label{fig:mnist4gap}
\end{figure}

Using the same extra-layer decoder design that worked well in the previous application, to build a 16-codeword representation we use a network with architecture $16\to 256\to 256$. Figure \ref{fig:mnist4gap} shows the evolution of the gap when training on 1024 data (in a single batch) with hyperparameters $\sigma=5$, $\beta=0.2$ and $\gamma=10^{-3}$. The gap has saturated at a value above 0.4. There is clearly no point in performing more iterations.

Recall that in the infeasible case one of the proximal points is on set $A$, the other on set $B$. In the divide-and-concur setup, the point on $B$ comes closest to what would be considered an approximate solution because the weights concur across all the data items and the network with those weights is a representation in the usual sense. The companion proximal point, on set $A$, has the property that the network's 1024 instantiations has 1024 outputs, one for each data item. For the point on $B$ to be close to the point on $A$, the 16 ``archetype'' data vectors generated by the concurring weights (of point $B$, from each of the 16 codewords) try to be close to those 1024 data vectors. The proximal pair found by RRR corresponds to an optimized partition of the 1024 data into 16 subsets, with the data in each subset close to one of the 16 data generated by the model. The quality of the partition can be evaluated by the average reconstruction accuracy (fraction of bits in the images of the partition matching the bits of the archetype). The accuracy of the method just described is 81.0\%. For the same task, the binary $k$-modes algorithm \cite{huang1998extensions}
with $k=16$ achieves 82.6\% (details in Supplemental Materials).

Figure \ref{fig:mnist4pix} shows the 16 4's generated by the model. \texttt{boolearn} also outputs how often each codeword was used in the approximate solution. The images in the figure are ranked by decreasing frequency, from 99 times at the top left, to 23 times at the bottom right. The shapes of these archetype 4's are qualitatively reproduced when RRR is rerun with a different random initialization, as are their frequencies. The highly slanted form is always the least frequent.

\begin{figure}[t!]
\centering
\includegraphics[width=\columnwidth]{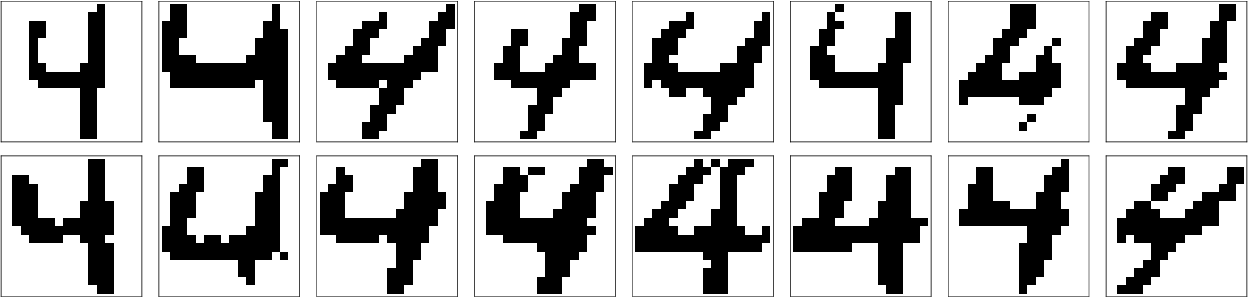}
\caption{The 16 archetype 4's generated by the 1-hot codeword model when trained on 1024 examples. The images are ranked by frequency, from top-left (highest) to bottom right (lowest).}
\label{fig:mnist4pix}
\end{figure}

\subsection{Analog data and MNIST}
\label{sec:mnist-analog}

We use the MNIST dataset to show that the discrete BTF outputs are not an impediment to learning analog data. This exercise also provides another example of the power of the margin hyperparameter and letting proximal point-pairs be the endpoint of learning.

To enhance the analog character of the data we use the down-sampled $8\times 8$ images available at \cite{alpaydin1998opticaldigits}. Samples of the ten digits are shown in Figure~\ref{fig:mnistjrdata}. The perceptron model with architecture $64\to 10$ is already interesting in that our method uses a very different approach to learning the weights in the model. In this no-hidden-units network there are just 10 BTFs, one for each digit class. The data constraints \eqref{eq:datacon} impose continuous $y$ values at the input nodes in the $A$ constraint and $\pm 1$ class-encoding values on the $y$'s at the output nodes in the $B$ constraint. The single correct class node has value $+1$ and the nine others are set at $-1$.

\begin{figure}[b!]
\centering
\includegraphics[width=0.8\columnwidth]{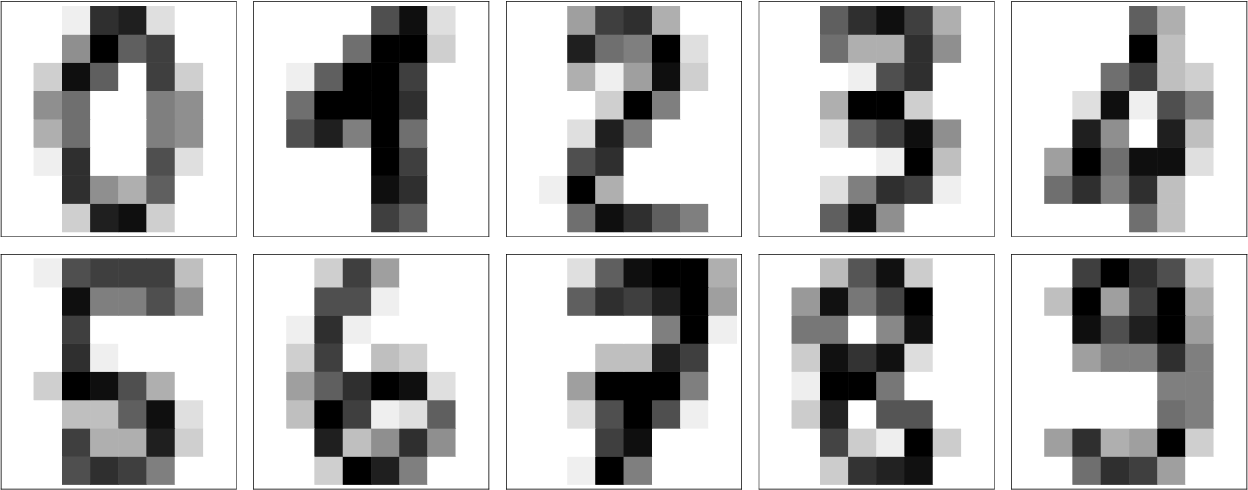}
\caption{Down-sampled MNIST digits \cite{alpaydin1998opticaldigits}.}
\label{fig:mnistjrdata}
\end{figure}

Figure~\ref{fig:mnistjrgap} shows the evolution of the gap over $2\times 10^4$ iterations when training on 128--2048 items. Since the algorithm is not faced with a combinatorial challenge (multiple near-solutions) we set $\gamma=0$. Also, by using the small time-step $\beta=0.05$ there is better convergence to proximal point-pairs---no hurry to get extricated from near-solution traps. The margin hyperparameter, given by $\mu=\sqrt{65/\sigma}$, has the expected effect on the gap. A small margin, with $\sigma=100$ (left panel), puts a feasible point almost within reach and the gap attains small values. When the margin is large, with $\sigma=1$ (right panel), the algorithm is forced to find a proximal point-pair.

\begin{figure}[t!]
\centering
\includegraphics[width=\columnwidth]{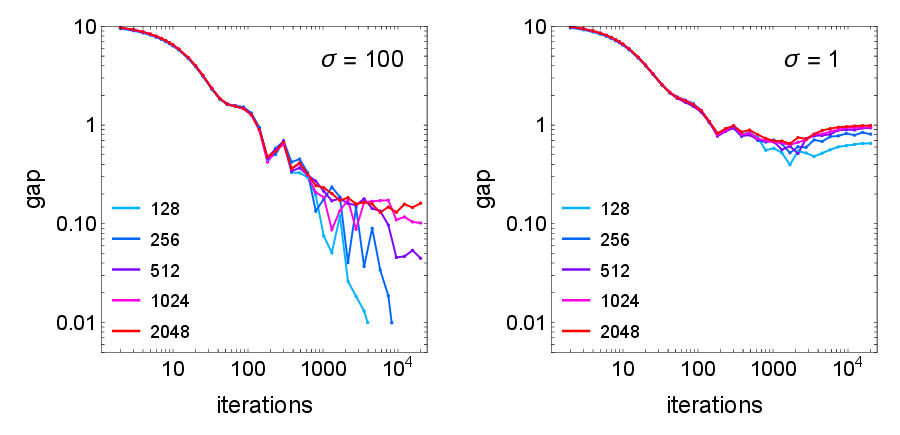}
\caption{Effect of the $\sigma$ hyperparameter on the gap when learning to classify MNIST data for various numbers of training items.}
\label{fig:mnistjrgap}
\end{figure}

\begin{figure}
\centering
\includegraphics[width=\columnwidth]{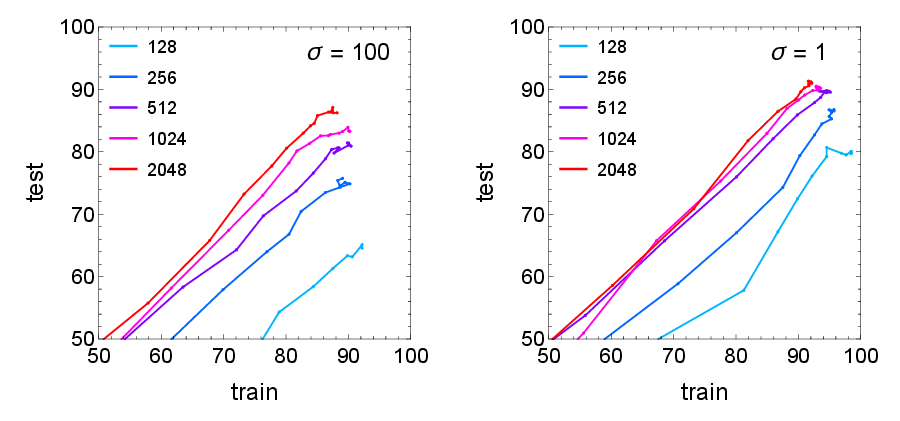}
\caption{Detail of the classification-accuracies corresponding to the training runs in Figure~\ref{fig:mnistjrgap}.}
\label{fig:mnistjracc}
\end{figure}

The corresponding accuracy plots in Figure~\ref{fig:mnistjracc} make the case that finding a good proximal point-pair is the superior strategy for this kind of data. Recall that the point on the $B$ constraint defines the solution because the weights concur across all the training items. The class prediction is defined by the output-BTF having the largest value of $w\cdot y$, where $y$ is the vector of grayscale pixel values in a train/test item. Both plots show the learning curves moving toward the main diagonal as the number of training items is increased. But the test accuracy is in general significantly greater for $\sigma=1$ than it is for $\sigma=100$. Increasing the BTF margin---by a factor of 10---thereby making the constraint problem infeasible, improved generalization by increasing the separation between the true class and all the wrong classes.

The learned weights, shown in Figure~\ref{fig:mnistjrwgts}, are also interesting. Negative values, shown as red, are just as prominent in the centers of the images as the positive values in blue. Apparently good generalization is just as much about eliminating wrong classes as identifying the correct one. Every detail of Figure~\ref{fig:mnistjrwgts} is reproduced when the algorithm is rerun from a different random start. The optimal weights for the BTF-perceptron enjoy the same uniqueness, empirically, as the weights obtained in convex optimization. 

\begin{figure}[t!]
\centering
\includegraphics[width=0.8\columnwidth]{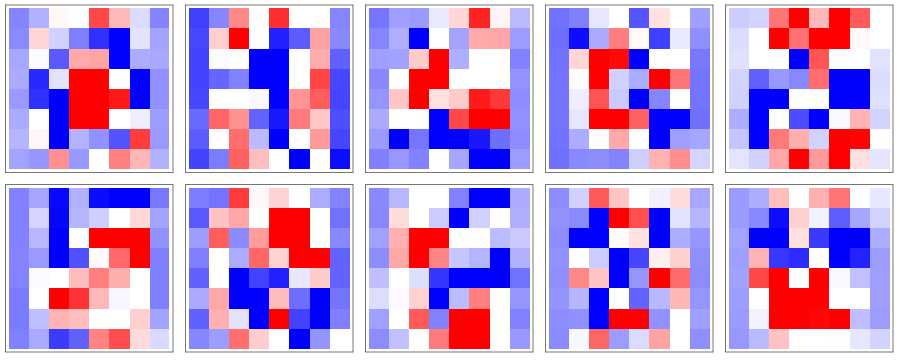}
\caption{The learned MNIST weights for the ten digit classes (in the same order as the samples in Figure~\ref{fig:mnistjrdata}) when training on 2048 items with $\sigma=1$. Blue/red denote positive/negative values.}
\label{fig:mnistjrwgts}
\end{figure}

\subsection{Logic circuits}
\label{sec:randlogic}

We finally come to the kinds of applications where we believe BTF networks can have the greatest impact. The foregoing applications served to contrast the new and existing approaches, while also showing the constraint-based method is not all that different from the perspective of the user. In this section we turn to applications where Boolean representations are natural and have the most to gain from implementations by BTF networks.

Learning Boolean-native data was previewed at the end of Section \ref{sec:overview}. There the BTF networks served as ``Boolean inference scratchpads'' on which logical rules connecting inputs and outputs were reconstructed. As in any scratchpad, the reconstructions are not unique. For example, logical inferences from the lower layers can simply be copied to higher layers for later use. Given enough layers, arbitrarily complex Boolean functions can be expressed on such scratchpads.

The obvious application that can take advantage of a Boolean representation is language. Words are built from a stem, or root, and various affixed morphemes that modify the meaning and grammatical function. The latter do not form a continuum, but are either absent or present. Clearly a \texttt{word2bits} embedding strategy is more natural than continuous embeddings, like \texttt{word2vec}. Moreover, if the ``atoms'' of language are Boolean, then grammar and phrase structure are best represented by Boolean functions. At the highest level is the task of reasoning---the original motivation for the whole Boolean formalism.

The only thing that has kept representations from being Boolean is the reliance of learning algorithms on gradient-based optimizers. For the RRR optimizer, acting on the divide-and-concur variables of a BTF network, discreteness of the representation does not pose a problem. In this section we study just this technical question and leave language modeling for the future.

The circuits shown in Figure \ref{fig:randlogic} implement Boolean functions of the form $f:\{0,1\}^m\to\{0,1\}^m$ that are compositions of simple ``gates'' in $\ell$ layers. BTFs are linearly separable gates, and the $\sigma$ parameter controls their complexity. We used $\sigma=3$ in our experiments, which limits the BTFs to the 1-input \textsc{Copy} gate, 2-input \textsc{And/Or}, or 3-input \textsc{Maj}. Our circuits were drawn from uniform distributions so that for large $m$ two random instances have the same behavior (and it suffices to study just one). We considered two distributions, both of which used the \textsc{Copy} gate with probability 1/2. The other gates were always \textsc{And/Or} in one distribution, \textsc{Maj} in the other. Gate-input nodes from the layer below were drawn from the uniform distribution (of singles, doubles, or triples) with one proviso: no dead-ends in the layer below. This just means the Boolean value at every node gets used by some gate in the layer above. \textsc{Not} was applied with probability 1/2 at every edge (gate-input), including the constant network-input used by the \textsc{And/Or} gates.

The learning curves plotted in Figure \ref{fig:accplots} suggest there is a transition to perfect generalization when the number of training items $N_\mathrm{data}$ exceeds some threshold. We can test this hypothesis by comparing the number of circuits with parameters $(m,\ell,\sigma)$ and the bits of information provided by $N_\mathrm{data}$ items. The number of circuits $N_\mathrm{circuit}$ allowed by the margin constraint with $\sigma=3$ is easy to calculate when the gate inputs are sufficiently diverse so that condition \eqref{eq:supportbound} holds:
\begin{equation*}
N_\mathrm{circuit}(m,\ell,3)=\left(2^1\binom{m+1}{1}+2^3 \binom{m+1}{3}\right)^{\ell m}\;.
\end{equation*}
The two terms correspond to the $n=1$ and $n=3$ relevant inputs allowed by $\sigma=3$ and the $m+1$ node choices include the constant network-input node (used by \textsc{And/Or}).

The number of distinct Boolean functions $f$ that can be implemented on a circuit with parameters $(m,\ell,3)$ is smaller than $N_\mathrm{circuit}(m,\ell,3)$ because implementations are far from unique. The greatest source of multiplicity comes from the fact that the nodes within each hidden layer can be freely permuted (without changing edge connectivity) and negation, when applied to all of a hidden BTF's inputs and outputs, also has no effect on $f$. Together, these symmetries give the following lower bound on the multiplicity:
\begin{equation}\label{eq:circuitmult}
N_\mathrm{mult}(m,\ell)=\left(m!\; 2^m\right)^{\ell-1}\;.
\end{equation}
The number of bits of information required to specify our class of Boolean functions is therefore
\begin{equation*}
h(m,\ell,\sigma)=\log_2 N_\mathrm{circuit}(m,\ell,\sigma)-\log_2 N_\mathrm{mult}(m,\ell)\;.
\end{equation*}
We will treat this as an estimate since the first term is valid only for sufficiently diverse inputs while the multiplicity \eqref{eq:circuitmult} is only a good lower bound. It evaluates to
\begin{equation*}
h(32,5,3)=2466.5-598.7=1867.8
\end{equation*}
for the parameters in our experiments.

The training algorithm can only acquire these bits of information from the values of the $m$ outputs bits in each data item. But the information rate is less than $m$ bits per data item, since these are not uniformly distributed. For example, depending on the circuit, some bits may be constant or exactly correlated or anticorrelated with other bits. If $p(x')$ is the probability distribution of output bits $x'\in\{0,1\}^m$ given by $x'=f(x)$, when the inputs $x$ are uniformly sampled, the information per data item is given by the entropy formula:
\begin{equation*}
s(p)=-\sum_{x'\in\{0,1\}^m}p(x')\log_2 p(x')\;.
\end{equation*}
From the distributions $p$ produced by feeding $10^5$ uniform input samples into the two 5-layer, width-32 circuits described above, we obtained
\begin{align*}
s(\mbox{\textsc{And/Or}})&=10.6\\
s(\mbox{\textsc{Maj}})&=14.2\;.
\end{align*}
We then arrive at the following estimates for the number of data items required to specify the two circuits:
\begin{align*}
h(32,5,3)\;/\;s(\mbox{\textsc{And/Or}})&=176\\
h(32,5,3)\;/\;s(\mbox{\textsc{Maj}})&=132\;.
\end{align*}
The transition (Figure \ref{fig:accplots}) to apparent perfect generalization falls between 128 and 256 data items for both data sets, consistent with these estimates.

Of course without access to testing data one cannot detect a transition from plots such as Figure \ref{fig:accplots}. Plots of gap vs. iterations can be used instead. Figure \ref{fig:a5gaptrans} shows on the left the evolution of the gap with increasing data (32, 64, 96, 128) when the amount of data is insufficient for perfect generalization. Increasing the data makes the under-constrained problem harder, and more iterations are required to reach a low gap. We have plotted ``min-gap'', the lowest currently achieved gap, as this does a better job conveying the evolution. The plots on the right show the evolution when the amount of data is sufficient for perfect generalization (256, 320, 384, 448). Initially (few iterations) the gap is increased with an increase in the data. However, we see that eventually this trend is reversed: increasing the data makes the over-constrained problem easier!

\begin{figure}[t!]
\centering
\includegraphics[width=\columnwidth]{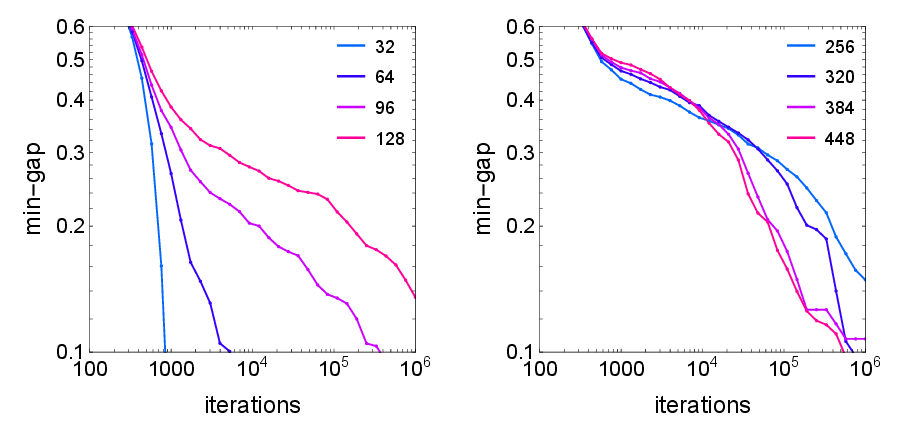}
\caption{Minimum value of the gap achieved with increasing iterations for the 5-layer \textsc{Copy/And/Or} circuit data. On the left is shown the evolution with increasing data in the under-constrained case. As shown on the right, the evolution changes character in the over-constrained case.}
\label{fig:a5gaptrans}
\end{figure}

We should not lose track of the fact that good generalization hinges very much on the choice of model. In this application the model is completely specified by just three parameters: $(m, \ell, \sigma)$. We believe that decreasing any of these from $(32, 5, 3)$ makes the constraint problem infeasible with sufficiently many data, and there would be no transition of the kind described above. Moreover, the quality of generalization is independent of the constraint-solver's hyperparameters, of which there are just two: $\beta$ and $\gamma$. These control the local dynamics of the search (depicted in Figure \ref{fig:convergence}), and not so much the endpoint of the search. In this application we used $\beta=0.2$ and $\gamma=10^{-3}$ in all the experiments.

\subsection{Cellular automata}
\label{sec:cellauto}

In the Boolean scratchpads we might use to learn some data, there is a tradeoff between width and depth (number of layers). A nice example of this tradeoff is the learning of cellular automata data.
Automata data differs qualitatively from  random circuit data in that each output bit is a function of just a small number of the input bits. This property may explain our finding that a sparsity-enhancement of the gradient-descent method has a striking positive effect.

Here we consider automata on a 1D periodic world where each cell applies a particular Boolean rule \textsc{R} to itself and the two cells adjacent:
\begin{equation*}
x(p,t+1)=\mbox{\textsc{R}}\big(\,x(p-1,t)\,,\,x(p,t)\,,\,x(p+1,t)\,\Big)\;.
\end{equation*}
The positions $p$ are periodic with period $L$ (the size of the world) and $t$ is the time. Our results are for Wolfram's ``chaotic'' rule-30 with formula
\begin{equation*}
\mbox{\textsc{R}}_{30}(x,y,z)=x\oplus(y\lor z)\;,
\end{equation*}
where $\oplus$ and $\lor$ are respectively the exclusive and regular \textsc{Or} operations.

Many of Wolfram's $2^8$ rules (on three inputs) can be expressed by a single BTF with negations, but rule-30 is not one of them. Expanding by the two \textsc{True} (or \textsc{False}) cases of the exclusive \textsc{Or}, rule-30 can be expressed in terms of \textsc{And/Or} as follows:
\begin{align}\label{eq:R30a}
\mbox{\textsc{R}}_{30}(x,y,z)
&=\mbox{\textsc{Or}}\Big(
    \mbox{\textsc{And}}\big(x,-\mbox{\textsc{Or}}(y,z)\big), \nonumber\\
&\qquad\qquad
    \mbox{\textsc{And}}\big(-x,\mbox{\textsc{Or}}(y,z)\big)
  \Big)\\
&=\mbox{\textsc{And}}\Big(
   -\mbox{\textsc{And}}\big(x,\mbox{\textsc{Or}}(y,z)\big), \nonumber\\
&\qquad\qquad
   -\mbox{\textsc{And}}\big(-x,-\mbox{\textsc{Or}}(y,z)\big)
 \Big)\;.
\end{align}
We have switched to $\pm 1$ for the Boolean values, where negation is multiplication by $-1$. This can be expressed on a BTF network with two sets of intermediate values, the first pair holding (a copy of) $x$ and $\mbox{\textsc{Or}}(y,z)$, and the second pair the two \textsc{True} (or \textsc{False}) cases of the exclusive \textsc{Or}. To implement one time step of rule-30 on a world of size $L=16$, a 3-layer BTF network with architecture
\begin{equation}\label{eq:30step1arch}
16\to 32\to 32\to 16
\end{equation}
suffices.

The left panel of Figure~\ref{fig:30step1acc} shows the evolution of train/test accuracies for 64--160 training data on the architecture \eqref{eq:30step1arch}. It appears that 96 training data are sufficient to learn rule-30 given enough RRR iterations. As in the logic circuit data experiments we used $\beta=0.2$ and $\gamma=10^{-3}$. In the right panel are results, for the same number of steps as iterations, of the gradient-descent method. To our surprise, 100\% test accuracy was achieved with 96 and 128 training items! The evolution is strange and related to a sparsity enhancement that was also tried on the logic data but without success.

\begin{figure}[t!]
\centering
\includegraphics[width=\columnwidth]{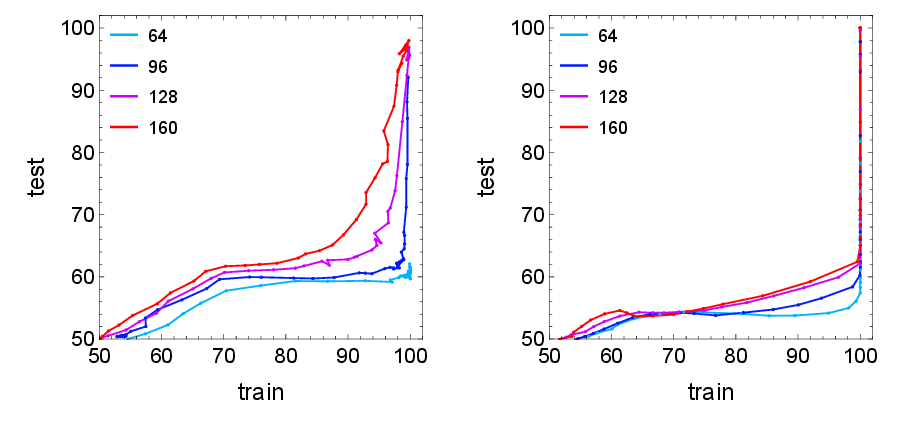}
\caption{Evolution of train/test accuracies for learning one step of the rule-30 automaton. The constraint-based results are on the left, gradient-descent with sparsity enhancement is on the right. Both methods used $10^6$ steps/iterations.}
\label{fig:30step1acc}
\end{figure}

\begin{figure}[b!]
\centering
\includegraphics[width=\columnwidth]{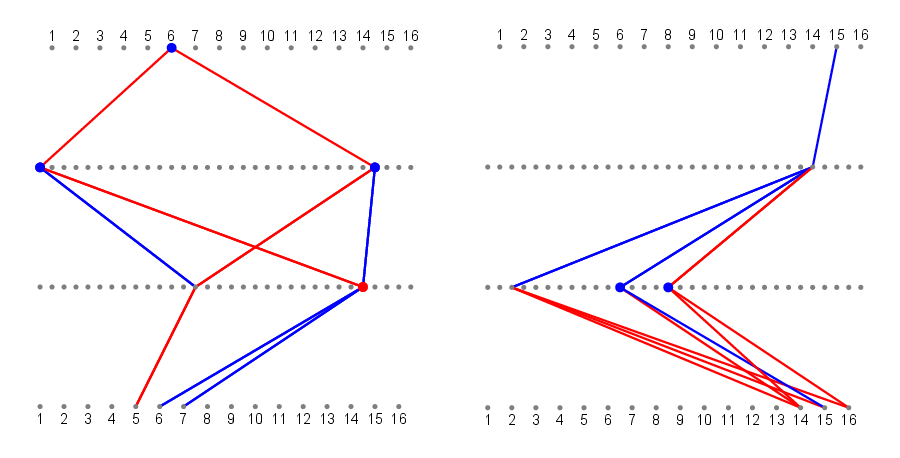}
\caption{By thresholding weights by magnitude, to identify the relevant BTF inputs, one can see rule-30 implemented in various ways.}
\label{fig:rule30traces}
\end{figure}

To promote sparsity one can add an $\ell_1$ penalty on the weights to the model's loss function. And since each neuron's input-weights should independently be sparse---like our ``support'' constraint---we also impose constraint \eqref{eq:weightnormalization}. Complete details are given in appendix~\ref{sec:gd_baselines}. By comparing the evolution of the two kinds of loss (binary-cross-entropy at the outputs, sparsity-penalty thoughout the network), it appears the optimizer first manages to perfectly reproduce the correct output bits---without regard to sparsity. In a longer second stage, with training accuracy holding steady at 100\%, the sparsity-penalty is reduced as well, until rather abruptly the rule is discovered.

Interpreting the solution-weights to read out the auto\-maton rule is easy in the constraint-based method. This is shown in Figure~\ref{fig:rule30traces} where, on the left, all the weights above some threshold in magnitude are traced that descend from output bit 6. That includes weights to the network's constant-input node (not shown). When above the threshold and positive, the node is colored blue and the BTF at that node implements \textsc{And}. Red nodes have the other sign of weight to the constant-input node and implement \textsc{Or}. The weight-trace shows that output node 6 only depends on input nodes (5, 6, 7),  and the Boolean function being implemented exactly matches \eqref{eq:R30a} after the three edges incident to the single \textsc{Copy} node are negated.

Interestingly, from the trace of node 15, shown in the right panel of Figure~\ref{fig:rule30traces}, we see that rule-30 can also be implemented as
\begin{equation*}
\mbox{\textsc{R}}_{30}(x,y,z)=\mbox{\textsc{Maj}}\big(-\mbox{\textsc{Maj}}(x,y,z),\mbox{\textsc{And}}(-x,y),\mbox{\textsc{Or}}(x,z)\big)
\end{equation*}
after two applications of De Morgan's rule (to flip the colors of edges). In this implementation the width is greater by one (three nodes instead of two), while the depth is reduced by one (only two layers). The architecture 
$16\to 48\to 16$ could therefore have been used instead. Indeed, learning rule-30 on this architecture is somewhat easier for our method.

\begin{figure}[t!]
\centering
\includegraphics[width=\columnwidth]{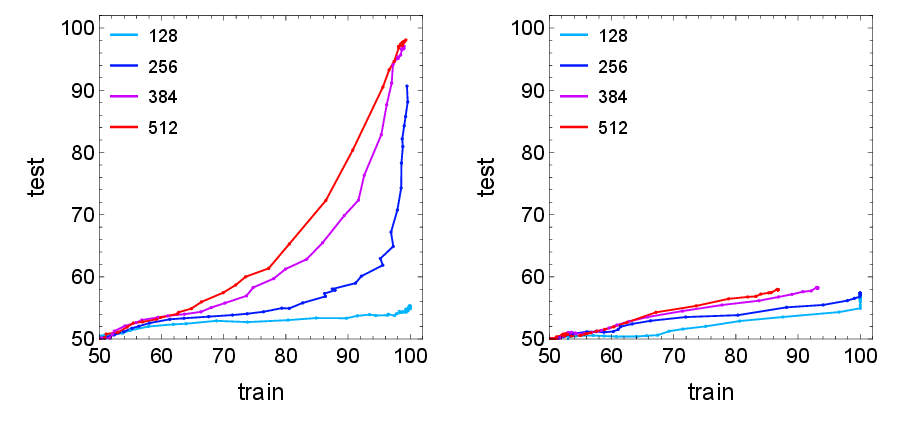}
\caption{Same as Figure~\ref{fig:30step1acc} but for data derived from two steps of rule-30.}
\label{fig:30step2acc}
\end{figure}

Learning rule-30 from input-output patterns separated by two steps of the rule is harder because each bit in the output is now a function of five input bits. From what we just learned about implementing one step of the rule, implementing two steps should be possible with architecture
\begin{equation*}
16\to 48\to 48\to 48\to 16\;.
\end{equation*}
Figure~\ref{fig:30step2acc} shows how the two methods compare on data from two steps of the rule. The iteration/gradient-step limit was doubled relative to the 1-step data on architecture \eqref{eq:30step1arch}, and sparsity enhancement was used as before in the gradient-descent method. The gradient-descent method is now struggling to even achieve 100\% training accuracy.

\section{Next steps}
\label{sec:next}

Gradient-descent on a loss-objective, and iterated projections to divided BTF constraints, are two strategies for learning with networks. The work for both methods scales in the same way:
\begin{equation*}
\frac{\mbox{computations}}{\mbox{gradient-step/RRR-iteration}}
\;\propto\;
E\times N_{\mathrm{data}}\;.
\end{equation*}
Here $E$ denotes the number of network edges and $N_{\mathrm{data}}$ the number of data items. The number of gradient steps, or the number RRR iterations, needed for learning depends on the nature of the data, and not much can be said about that. On the other hand, because the constraint approach is able to use BTFs, that strategy has the upper hand when we want the representation to be Boolean.

The advantage of a Boolean representation cannot be overstated. When implemented with BTF networks, and a sufficiently large margin, the BTFs are equivalent to BTFs with very simple weights! As explained in Section \ref{sec:replogicgate}, even for the generous support parameter $\sigma=9$, the equivalent BTFs all have $\pm 1$ weights with just one exception. Most of the power consumed by data centers is for inference---e.g. ``In \textit{Learning representations ...} did Rumelhart et al. use data batches?'' \footnote{It appears MS-Copilot did not read the Nature paper! According to Copilot, weights were updated after each example, but the paper clearly states the summed gradient of all the examples was used in the updates.}---and this would be reduced enormously if multiplications are eliminated. Inference in networks with $\pm 1$ weights uses only integer addition.

In training, the tighter hypothesis classes made possible by low-complexity (large-margin) BTFs promotes generalization. Power consumption is reduced because the same data can be learned from fewer examples. Strategies like ``dropout'' \cite{srivastava2014dropout} for avoiding overfitting in gradient-trained networks seem barbaric in comparison. 

Because the constraint-alternative operates at the lowest level of machine learning technology, many of the advances at the higher levels---all trailblazed with gradient-descent---might go through mostly unchanged. Because the representations will become Boolean, there will surely be differences. We hope this article has shown that a core technology based on constraints is at least worth exploring. Below are what we see as the next steps in functionality, implementation, infrastructure, and theory.

\subsection{Functionality}\label{sec:functionality}

\texttt{boolearn} is demonstration software and its programs were deliberately kept simple for that reason. Below are four features that will improve its functionality:

\textit{Weight sharing}\; This feature is easily implemented in the $B$ constraint. Weights are always required to concur across data, but additional equalities can be imposed for other reasons, as in convolutional filters. 

\textit{Nonuniform margins}\; There is no  reason to impose the same margin constraint on all the BTFs of the network (as in this work). For example, the margins in the encoder and decoder stages of the autoencoder (Section \ref{sec:autoencoder}) should really have been set to different values.

\textit{Batch mode}\; Whereas the constraint approach to optimization does not rely on ``stochasticity'' conferred by small data batches, the sheer size of data in many applications makes training---on as many network instantiations as there are data---completely infeasible. A practical solution is to use the largest size batch that can be implemented in hardware, and update the batch in each iteration by replacing the longest residing data item with a new item. With enough iterations all the data ends up getting used, in a cyclic order, where each item's residence time in the optimizer is equal to the size of the batch.

\textit{Noise accommodation}\; In phase retrieval, where the data constraints are derived from continuous measurement values (X-ray intensities), solutions are always near-solutions with exceptionally small gaps. The small-gap strategy also works when the data is discrete, and only few bits are corrupted. But there is a better alternative. The trick is to make an allowance for some fraction of the data bits to be flipped from their correct values. In the projection to the data constraint one identifies all the $y$ variables that are closer to the flipped bit. Of these the allowed fraction are sorted by absolute value and the largest (who change least when the flipped value is indeed the correct one) are projected to the flipped bit. Not only does this make infeasible problems feasible, the corrupted data is flagged in the solution.

\subsection{Implementation}

The current version of \texttt{boolearn} runs on a single core and all the programs are written in C. Before being considered an alternative to gradient-methods, the following are necessary steps:

\textit{Distributed processing}\;
The independence of the variables associated with each node in both constraint projections, together with the existing low-level C implementation of these projections, makes the framework naturally amenable to parallel and distributed execution. This structural separability has been discussed across multiple formulations and analyses \cite{elser2021learning,lal2023nonconvex,lal2025flow,bergmeister2025projection,lal2025backpropagation}. We plan to provide an explicit distributed implementation, and we expect that a variety of parallel and distributed realizations will emerge, tailored to different application domains and hardware environments.

\textit{Python interface}\;
Once distributed execution is in place, the software will be packaged with a Python interface (and potentially interfaces for other languages) to provide a user experience comparable to that of modern gradient-based learning frameworks.  In such an interface, the primary optimization diagnostic would be the \emph{gap} rather than a loss value, the RRR time step $\beta$  plays the role of the learning rate, and the metric relaxation hyperparameter $\gamma$ is roughly analogous to the moment decay rates of the \texttt{Adam} optimizer.  Aside from these differences, most user-facing concepts—such as network architectures, batch size, training schedules, and accuracy metrics—would remain closely aligned with standard machine learning practice.

\subsection{Infrastructure}

We anticipate the development of the following tools to become part of the BTF-learning ecosystem:

\textit{Model interpreters}\; As explained above, a BTF network trained with a large margin, or equivalently a small $\sigma$, is equivalent to a BTF network where all nonzero weights are $\pm 1$. Software that inputs a continuous-weight model and outputs an equivalent model with $\pm 1$ weights is a model interpreter. As explained in connection with the logic circuit reconstructions at the end of Section \ref{sec:overview}, such an interpreter faces a mild complication when the BTF inputs have exact correlations. In that case, interpreters not only have to identify a BTF's relevant inputs (Section \ref{sec:BTF}) but be able to detect correlations in their values.

\textit{Word embedding}\; As already explained in Section \ref{sec:randlogic}, \texttt{word2bits} makes a lot more sense than \texttt{word2vec} when the representation is Boolean. Most of the strategies behind \texttt{word2vec} carry over to embeddings on the Boolean hypercube.

\textit{Benchmarks for theory}\; Benchmark data sets are important for keeping machine learning developments in line with the needs of industry. It would be useful to additionally have benchmarks designed to resolve theoretical questions that arise in Boolean representations. How many dimensions are best for embedding the tokens for language-like data? What is the best tradeoff between network width and depth? It's important that ``best'' is evaluated not just on the basis of accuracy but also computational efficiency.

\subsection{Theory}

In their conclusions to ``Learning
internal representations ...''\cite{rumelhart1985learning}, Rumelhart et al. quote Minsky and Papert on the challenges of establishing ``an interesting learning theorem'' for  multilayered architectures. However, they sum up their own work as follows:
\begin{quote}
Although our learning results do not \textit{guarantee} that we can find a solution for all solvable
problems, our analyses and results have shown that as a practical matter, the error propagation
scheme leads to solutions in virtually every case.
\end{quote}
This is also where things stand with Boolean threshold function networks and the very different approach we have developed for training them. Despite nonconvexity, the RRR gap is observed to vanish just as consistently, in diverse applications, as the loss in those gradient experiments 40 years ago. There may never be a rigorous ``learning theorem'' for either approach, but that should not discourage the development of theoretical analyses that at least make our experience with these methods plausible.

Much work, more than can be reviewed here, has tried to shed light on the gradient-descent process. Many of the same questions can be asked about RRR. Is learning incremental? It certainly seems that way from the quasi-monotone behavior of the gap. The interpretability of the end product of learning with BTF networks makes the study of such questions especially interesting.

\begin{acknowledgments}
The authors are grateful to Heinz Bauschke for his influence on their thinking about projection methods and for valuable feedback. V.E. thanks Chris Myers and Nick Elyu for technical support. M.K.L. acknowledges financial support from the Alexander von Humboldt Foundation at the Technical University of Munich (TUM) through May 2026 and thanks Suvrit Sra, Adrien Taylor, and Francis Bach for helpful discussions.
\end{acknowledgments}

\section*{Data availability}
Data and code supporting the findings of this study are available at \href{https://github.com/veitelser/boolearn}{\texttt{boolearn}}. 

\appendix
\section{Gradient-based baselines}
\label{sec:gd_baselines}

Backpropagation-based neural networks have long been used to model logical
structure; early work already noted that saturating sigmoids can approximate
Boolean behavior \cite{rumelhart1985learning,rumelhart1986learning}. Subsequent
work introduced mechanisms for discrete or near-discrete computation within
gradient-based training, including straight-through gradient estimators and
binary-weight networks \cite{bengio2013estimatingpropagatinggradientsstochastic,
courbariaux2015binaryconnect}, as well as differentiable logic-gate architectures
that relax discrete choices during training and discretize afterward
\cite{petersen2022deep}.

This appendix documents the gradient-based baselines used for comparison with
BTF networks trained using RRR. The purpose of these baselines is not to develop
specialized differentiable circuit learners, but to test whether a standard,
fully connected multilayer perceptron trained with widely adopted gradient-based methods can
recover the input--output behavior of our Boolean benchmarks. A public leaderboard for the benchmark is
available \cite{random-majority-circuit-learning}.

Our baseline configuration follows common practice in modern deep learning:
fully connected ReLU networks trained with \texttt{AdamW}~\cite{loshchilov2017decoupled}, using standard hyperparameters. This setup is consistent with recent empirical
studies of Boolean-function learnability that employ comparable architectures
and optimization settings \cite{nicolau2025understanding}.

In addition, we tested (i) plain \texttt{Adam}~\cite{kingma2014adam} without decoupled weight decay and
(ii) a projected-gradient variant enforcing hard row-norm and row-sparsity
constraints after each update. In our experiments, neither modification
improved test performance relative to the reported baselines, and they are
therefore not included in the main comparison.

\textit{Relation to quantization-aware training}\;
Although our method (BTF networks) uses discrete internal values, it is not a variant of
quantization-aware training (QAT). In QAT, discretization is introduced during
training via surrogate gradients—typically for deployment efficiency—while the
underlying optimization remains loss-based and continuous.
Representative examples include PACT \cite{choi2018pact}, learned step-size
quantization (LSQ) \cite{esser2020lsq}, vector-quantized VAEs
\cite{van2017vqvae}, and post-training quantization methods for transformer
models \cite{frantar2023gptq,xiao2023smoothquant,dettmers2023qlora}.
By contrast, our method imposes Boolean node values and hard threshold constraints
from the outset, and learning is formulated as a feasibility problem with
explicit margin and normalization constraints. Discreteness therefore serves as
a structural inductive bias for rule and circuit discovery, not as a compression
mechanism.

\subsection{Model class and training objective}
\noindent For each dataset we observe input–output pairs
\begin{equation*}
\mathcal{D}
=\{(x_i,y_i)\}_{i=1}^{M},
\qquad
x_i\in\{0,1\}^{d_{\mathrm{in}}},
\quad
y_i\in\{0,1\}^{d_{\mathrm{out}}}.
\end{equation*}
For gradient-based baselines, inputs and outputs are encoded as real values in $\{0,1\}$. We fit a fully connected multilayer perceptron
$f_\theta:\mathbb{R}^{d_{\mathrm{in}}}\to\mathbb{R}^{d_{\mathrm{out}}}$
with $L$ affine layers:
\begin{equation*}
\begin{aligned}
h_0(x) &= x,\\
h_\ell(x) &= \mathrm{ReLU}\!\left(W^{(\ell)}h_{\ell-1}(x)+b^{(\ell)}\right),
\quad \ell=1,\dots,L-1,\\
f_\theta(x) &= W^{(L)}h_{L-1}(x)+b^{(L)},
\end{aligned}
\end{equation*}
where $\theta=\{W^{(\ell)},b^{(\ell)}\}_{\ell=1}^L$ and $w^{(\ell)}_j$ denotes the $j$-th row of $W^{(\ell)}$.
The final layer is linear and outputs logits; no sigmoid activation is applied inside the network.

The layer widths and depths for each task are specified in the corresponding experimental sections~\ref{sec:randlogic}, and \ref{sec:cellauto}. All hidden layers use ReLU activations. No batch normalization, dropout, or other architectural modifications are employed.

Let $N_\mathrm{data}$ denote the number of training examples. We define the
(binary) cross-entropy loss with logits (BCE) as
\begin{equation*}
\mathcal{L}_{\mathrm{BCE}}(\theta)
=
\frac{1}{N_\mathrm{data}\,d_L}
\sum_{i=1}^{N_\mathrm{data}}
\sum_{k=1}^{d_L}
\Big(\log(1+e^{z_{ik}})-y_{ik}z_{ik}\Big),
\end{equation*}
where $z_i=f_\theta(x_i)$ are the output logits of the network.
This is the average per-output-bit negative log-likelihood over the training set.

\subsection{\texttt{AdamW} and \texttt{Penalty-AdamW}}
\label{sec:AdamW_penalty}

All gradient-based baselines are trained using \texttt{AdamW}
\cite{loshchilov2017decoupled}, i.e.\ \texttt{Adam} with \emph{decoupled weight decay}.
At each optimization step $t$, we compute the full-batch gradient
$g_t=\nabla_\theta \mathcal{J}(\theta_t)$ on the \emph{entire} training set and
apply one \texttt{AdamW} update with learning rate $\eta$, moment-decay parameters
$(\beta_1,\beta_2)$, numerical stabilizer $\varepsilon$, and decoupled weight
decay coefficient $\lambda_{\mathrm{decay}}$, which is applied to all parameters in our
implementation.

\textit{AdamW baseline}\;
The \texttt{AdamW} baseline minimizes the data-fit objective
\[
\mathcal{J}(\theta)=\mathcal{L}_{\mathrm{BCE}}(\theta),
\]
where $\mathcal{L}_{\mathrm{BCE}}$ is the binary cross-entropy with logits on the
training set.

\textit{Penalty-AdamW}\;
To bias optimization toward sparse, circuit-like weight structure, we augment
the training objective with a row-norm penalty and an $\ell_1$ penalty:
\begin{equation*}
\begin{aligned}
\mathcal{J}(\theta)
&=
\mathcal{L}_{\mathrm{BCE}}(\theta)
+\lambda_{\mathrm{norm}}
\sum_{\ell=1}^{L}
\frac{1}{d_\ell}
\sum_{j=1}^{d_\ell}
\Big(\|w^{(\ell)}_j\|_2^2-d_{\ell-1}\Big)^2\\
&\qquad\qquad\quad
+\lambda_{1}
\sum_{\ell=1}^{L}
\frac{1}{d_\ell d_{\ell-1}}
\sum_{j,k}
\lvert W^{(\ell)}_{jk}\rvert.
\end{aligned}
\end{equation*}

The row-norm term penalizes deviations from a fixed reference scale,
while the $\ell_1$ term promotes sparsity within rows.
Both penalties are applied only to weight matrices (biases are not penalized).
No projection step or feasibility enforcement is used; optimization remains
entirely gradient-based.

The row-norm penalty fixes a reference scale in an otherwise positively
homogeneous ReLU network. Because ReLU satisfies
$\mathrm{ReLU}(c x)=c\,\mathrm{ReLU}(x)$ for $c>0$, the parameterization is
degenerate under layerwise rescalings that leave the network function
unchanged. We therefore impose the target $\|w^{(\ell)}_j\|_2^2=d_{\ell-1}$
as a fixed, untuned scale convention across layers and datasets.
Its role is purely to remove this degeneracy and stabilize optimization.
Together with the $\ell_1$ term, it biases rows toward a small number of
large-magnitude entries, promoting sparse, gate-like structure.
Although the choice of $d_{\ell-1}$ parallels the fan-in normalization used
in the BTF framework, it carries no Boolean interpretation here and is used
only to control weight magnitudes and enable comparison across architectures.

\begin{figure*}
\centering
\includegraphics[width=2\columnwidth]{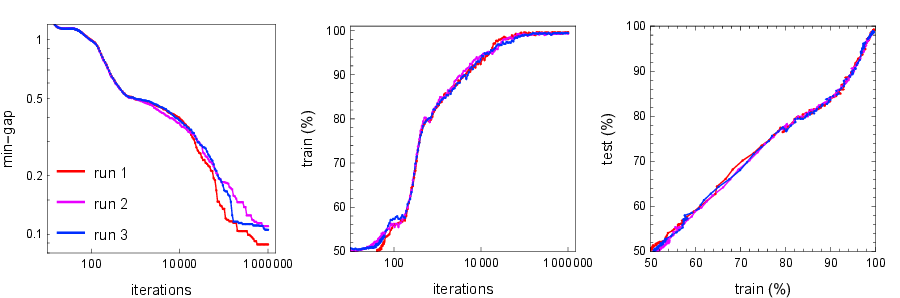}
\caption{Near reproducibility of the evolution of the minimum gap (left) and training accuracy (center) with iterations, in three runs of the RRR training algorithm differing only in the initialization. The right panel compares training and test accuracies. The same 512 data items derived from the circuit in Figure \ref{fig:randlogic} were used in all runs.}
\label{fig:initsensitivity}
\end{figure*}

\textit{Hyperparameters and initialization}\;
Unless stated otherwise, all runs use the standard \texttt{AdamW} settings commonly
adopted in modern deep learning (and also used in recent work with Boolean data \cite{nicolau2025understanding}): 
\[
\eta=10^{-3}\qquad
\beta_1=0.9\qquad
\beta_2=0.999
\]
\[
\varepsilon=10^{-8}\qquad \lambda_{\mathrm{decay}}=10^{-4}.
\]
For \texttt{Penalty-AdamW} we additionally fix
\[
\lambda_{\mathrm{norm}}=10^{-2}\qquad
\lambda_{1}=10^{-4}
\]
globally across all datasets and architectures (not tuned per task).

Weights are initialized using the
standard Kaiming uniform initialization for ReLU networks
\cite{he2015delving}. Concretely, for a layer $\ell$ with fan-in
$d_{\ell-1}$, weights and biases are sampled independently and uniformly from the interval with bounds $\pm 1/\sqrt{d_{\ell-1}}$. This initialization is applied uniformly across all datasets and
architectures without modification.

\textit{Batching, step budgets, and logging}\;
All gradient-based baselines use full-batch optimization:
each optimization step consists of one \texttt{AdamW} update using the gradient
computed on the entire training set. No data subsampling, reshuffling,
dropout, or data augmentation is employed.

Random-logic and Rule-30 (1-step) experiments are run for $10^6$
optimization steps, and Rule-30 (2-step) experiments for
$2\times 10^6$ steps. We do not use early stopping; all runs are
terminated after a fixed number of steps to ensure comparable
computational budgets across methods. Loss and accuracies are recorded at checkpoints whose indices are
geometrically spaced in iteration number (i.e., the interval between
checkpoints increases approximately exponentially). Test-set quantities are computed
only for evaluation and never enter the update rule.

\section{Initialization sensitivity}
\label{sec:initsensitivity}

A reviewer asked if training BTF networks with RRR has the same broad distribution of solution times (for different initializations) as some other highly nonconvex problems that the algorithm has been applied to. For example, packing 40 spheres into a 10-dimensional periodic cube (strategy of the densest known packing) is the proverbial needle-in-a-haystack \cite{elser2025solving}. Solution times have the same random distribution as radioactive decay. If our training method exhibited this behavior, it would be disqualified if evidence of incremental progress toward the solution was important, if only for peace of mind.

Figure \ref{fig:initsensitivity}
should answer the question. Shown are the gap, training accuracy, and generalization at 200 (logarithmically spaced) checkpoints in the Boolean function benchmark (data from the circuit in the left panel of Figure \ref{fig:randlogic}). The plots show these for three initializations (uniform samples between $-1$ and $+1$ for all the variables). Not only is the total-iteration budget for successful learning very reproducible, so is the evolution of the gap and accuracy over the entire course of training.

One should not look at Figure \ref{fig:initsensitivity} as evidence the constraints in the logic circuit benchmark are quasi-convex. The three solutions are completely different, as allowed by the symmetries of the problem (Section \ref{sec:sym}). The initialization merely selects one of very many symmetry-equivalent solutions. However, the question remains: Why does BTF-network training not exhibit needle-in-a-haystack behavior, with its broadly distributed epochs of chaotic search? This is a very good question to which we do not have an answer.

\begin{figure*}
\centering
\includegraphics[width=2\columnwidth]{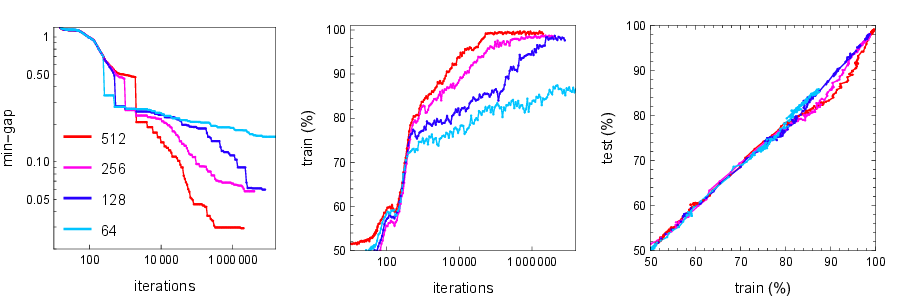}
\caption{Same as Figure \ref{fig:initsensitivity} but showing the batch-size dependence when training in batch mode. The runs with batch sizes 256, 128 and 64 were allowed 2, 4 and 8 times, respectively, as many iterations as the run in which all 512 data were processed in each iteration. All runs completed in the same amount of time (in our serial implementation) because the smaller batch sizes use less time per iteration.}
\label{fig:batch}
\end{figure*}

\section{Batch mode}
\label{sec:batchmode}

The same reviewer asked for a demonstration that our proposal in Section \ref{sec:functionality} for batch-mode training works as promised. We implemented this and also compared it to the  more obviously disruptive protocol where the entire data batch is replaced with fresh training data at regular intervals.

In both protocols the network instantiations for the new data are initialized by taking the data at the network inputs and propagating it to the outputs with the current estimate of the weights, the concurred variables $w_B$, as these are shared across instantiations. In general, this forward-pass initialization will satisfy all the constraints for its instantiation except the BTF margins and the equality of $y_A$ (from forward pass) and $y_B$ (from data) at the network's output nodes. Of course, when the correct weights have been discovered even these constraints will be satisfied with this initialization scheme. The metric coefficients of the variables for the new instantiations are initialized to 1.

Replacing whole data batches proved to be the better protocol, and we present results only for this. The batch-mode version of the training program has two new arguments. The argument \texttt{items} still specifies the number of training data, while now \texttt{batch} lets the user specify the number of these that the RRR constraint solver processes at any time. The second new argument, \texttt{batchiter}, tells RRR to perform this many iterations before replacing the current batch (of data and network instantiations) with a new one.
In batch-mode both the time per RRR iteration and the active memory are reduced by the ratio \texttt{batch}/\texttt{items}.

In order for runs to behave systematically with the setting of \texttt{batch}, we found that \texttt{batchiter} had to be proportional to \texttt{batch}. It appears, that if $N$ total iterations are invested to learn concurring weights, then the same number of training data have to be seen over the course of those $N$ iterations, regardless of how those data are divided into batches, each processed with its share of the $N$ iterations. The main downside of this strategy is that the metric coefficients have insufficient time to evolve when \texttt{batch} is small, as they are initialized to 1.

The results of the batch-mode experiments are shown in Figure \ref{fig:batch}, for \texttt{batch} settings 512, 256, 128 and 64 when training on 512 items. The iteration limit for the $\text{\texttt{batch}}=512$ run was set at $2\times 10^6$, and this was increased by factors of 2, 4 and 8 for the smaller batches so that all runs required the same total work (consumed energy). Because the run with $\texttt{batch}=512$ uses all 512 data constraints in each iteration, this only tests our initialization scheme. For this run we set $\text{\texttt{batchiter}}=2000$ and reduced this by factors of 2, 4 and 8 for the smaller batches as explained above.

Figure \ref{fig:batch} shows the same information as Figure \ref{fig:initsensitivity}, but now comparing the different batch sizes. For $\texttt{batch}=512$ we are reassured to see behavior that matches Figure \ref{fig:initsensitivity}, a sign that the forward-pass initialization scheme is working. Remembering that the plots should be compared with iterations scaled by factors of 2, 4 and 8, we see only a mild degradation in performance with batch size, except for the smallest batch. The behavior of that run may be due to the poor recovery of the metric coefficients (from initialzation to 1) when $\text{\texttt{batchiter}}=250$. The same metric relaxation parameter, $\gamma=10^{-3}$, was used in all runs.

This demonstration of batch-mode training is really only interesting for the two small batches, of size 128 and 64. As the results in Figure \ref{fig:accplots} show, a data set as small as 256 items is already sufficient for learning how to reconstruct this Boolean function. But $\texttt{items}=128$ provides insufficient information, so that batch-mode training on a larger data set is necessary.

\bibliographystyle{unsrtnat}
\bibliography{LWBTF}

\end{document}


\title{Supplemental material for Learning with Boolean threshold functions}

\begin{abstract}

\end{abstract}

\maketitle

\section{Overview}
This Supplemental Material reports additional comparisons and
robustness experiments added in response to the referees.  The common
model definitions, loss functions, initialization, and optimization
settings for the principal \texttt{AdamW} and \texttt{Penalty-AdamW} baselines are given
in Appendix~A of the main paper.  Here we describe only experimental
settings that differ from that protocol.

\section{Binary-network baselines for random-majority circuits}
\label{sec:supp-random-majority-binary-baselines}

The random-majority experiments in Section~II of the main paper
compare RRR with a real-valued multilayer perceptron trained by
backpropagation. Since the target circuits and the hidden states of
the BTF networks are Boolean, it is natural to ask whether the
observed difference is caused simply by the use of binary internal
representations.

We therefore supplement the principal real-valued ReLU baseline with
four gradient-trained variants. The first uses strict binary hidden
activations with real-valued weights and a straight-through
estimator. The second isolates equal-magnitude signed weights while
retaining continuous ReLU activations. The third combines binary
weights and binary hidden activations in a BinaryNet-style model. The
fourth replaces the hidden ReLU activation by the smooth
threshold-like function $\tanh$. These experiments separate the
effects of strict binary internal states, equal-magnitude signed
weights, and smooth thresholding.

The threshold and binary-network experiments use full-batch updates and run for \(10^6\) iterations. We report results at the final iterate, without using the test set for model selection or early stopping. In a separate experiment with $256$ training examples (Subsection~\ref{sec:supp-random-majority-minibatch}), we compare full-batch updates with cyclic and randomly reshuffled mini-batch updates.

\subsection{Threshold and binary-network baselines for random-majority circuits}

\subsubsection{Threshold and binary gradient models}

Binary activation represents the True/False values manipulated by a
logical computation. Logic imposes additional structure, however:
in symmetric unweighted gates such as \textsc{Maj}, \textsc{And},
and \textsc{Or}, the participating literals contribute with equal
strength. This motivates equal-magnitude signed weights, where the
weight sign determines whether an input is used directly or negated.

To make the role of the threshold explicit, consider the bipolar
encoding $x_i,y\in\{-1,+1\}$ and the threshold neuron
\begin{equation}
\label{eq:supp-logic-threshold}
    y
    =
    \operatorname{sgn}
    \left(
        \alpha\sum_{i=1}^{n}s_i x_i+b
    \right),~~
    s_i\in\{-1,+1\},
    ~\alpha>0.
\end{equation}
For unnegated inputs, the choices
\begin{align}
    b_{\textsc{And}_n}
        &=-(n-1)\alpha,\label{eq:supp-logic-biases1}\\
    b_{\textsc{Or}_n}
        &=(n-1)\alpha,\label{eq:supp-logic-biases2}\\
    b_{\textsc{Maj}_n}
        &=0
        ~ \text{for odd }n\label{eq:supp-logic-biases3}
\end{align}
implement the corresponding gates. In particular, for two-input
\textsc{And} and \textsc{Or}, the bias has the same magnitude as
each participating weight.

Under the diversity conditions of the BTF support result, the setting
$\sigma=3$ yields effective weight rows having support one or three.
Support one implements \textsc{Copy} or \textsc{Not}, while the
three-input equality case has equal-magnitude nonzero weights and
implements \textsc{Maj}, or \textsc{And}/\textsc{Or} when the fixed
input participates. Thus, equal-magnitude signed weights reproduce
one important feature of logical gates, but the number of nonzero
weights is also essential.

Standard binary-network methods generally retain real-valued latent
parameters and use straight-through derivatives for binary weights or
activations. Their effective activation thresholds may also remain
real-valued, for example through affine biases or the learned scale
and offset of batch normalization
\cite{courbariaux2016bnn,umuroglu2017finn}. Consequently, the dense
binary-weight models studied below reproduce the equal-magnitude
signed-weight property, but not the support-one or support-three
restriction or the gate-specific threshold relations of the BTF
model. We therefore treat them as established nearby baselines rather
than exact gradient implementations of the BTF hypothesis class.

Throughout the experiments below, the binary sign function uses the
convention
\begin{align*}
    \operatorname{sgn}(u)
    =
    \begin{cases}
        +1, & u\geq 0,\\
        -1, & u<0.
    \end{cases}
\end{align*}

\paragraph{Real weights and binary hidden activations.}

The first model uses real-valued weights and biases, but replaces the
activation at every hidden layer $\ell=1,\ldots,L-1$ by the binary
sign function,
\begin{equation}
\label{eq:supp-strict-ste-sign}
    h^{(\ell)}_j
    =
    \operatorname{sgn}
(a^{(\ell)}_j)
    \in\{-1,+1\}.
\end{equation}
The output layer remains real-valued and produces unrestricted logits.

Because the derivative of the sign function vanishes away from its
discontinuity, the backward pass uses the saturated straight-through
estimator
\begin{equation}
\label{eq:supp-ste-derivative}
    \frac{\partial h^{(\ell)}_j}
         {\partial a^{(\ell)}_j}
    \approx
    \mathbf{1}
    \left\{
        \left|a^{(\ell)}_j\right|\leq 1
    \right\}.
\end{equation}
This clipped estimator corresponds to backpropagating through the
derivative of the hard-tanh surrogate
$\operatorname{clip}(u,-1,1)$ and is commonly used in binarized
neural networks \cite{courbariaux2016bnn}. Gradients are passed
unchanged when $a^{(\ell)}_j\in[-1,1]$ and suppressed outside this
interval. Batch normalization is not used in this model.

\paragraph{Binary-weight ablation with ReLU activations.}

The second model isolates the effect of equal-magnitude signed weights
while keeping the hidden activations continuous. This property is
relevant to logic because the participating inputs of gates such as
\textsc{Maj}, \textsc{And}, and \textsc{Or} enter with equal
strength. The model does not, however, impose strict binary hidden
states or the support restriction of the BTF network.

For each layer $\ell$, let $\widetilde W^{(\ell)}$ denote the
real-valued latent weight matrix stored and updated by the optimizer.
During every forward pass, including both training and evaluation,
the latent matrix is replaced by
\begin{equation}
\label{eq:supp-binary-forward-weight}
    W^{(\ell)}
    =
    \frac{g_\ell}{\sqrt{d_{\ell-1}}}
    \operatorname{sgn}
    \left(\widetilde W^{(\ell)}\right),
\end{equation}
where $d_{\ell-1}$ is the fan-in of the layer. The gain is
$g_\ell=\sqrt{2}$ for the four hidden ReLU layers and $g_L=1$ for
the output layer. Since $d_{\ell-1}=32$ throughout the
random-majority architecture, the magnitude of every hidden-layer
forward weight is
\begin{align*}
    \frac{\sqrt{2}}{\sqrt{32}}
    =
    \frac{1}{4},
\end{align*}
while the magnitude of every output-layer forward weight is
$1/\sqrt{32}$. Real-valued biases are retained in all five affine
layers.

The derivative of the weight-binarization map is replaced by the
straight-through estimator in
\eqref{eq:supp-ste-derivative}. After each \texttt{AdamW} update,
every latent weight is restricted to the interval $[-1,1]$:
\begin{equation}
\label{eq:supp-latent-weight-clipping}
    \widetilde W^{(\ell)}_{jk}
    \leftarrow
    \max
    \left\{
        -1,\,
        \min
        \left\{
            1,\widetilde W^{(\ell)}_{jk}
        \right\}
    \right\}.
\end{equation}
The first four layers use real-valued ReLU activations, while the
final affine layer produces unrestricted real-valued logits.

The forward-pass binary weight matrices are dense by construction.
Consequently, this model does not reproduce the support-one or
support-three structure associated with the BTF setting $\sigma=3$.
It isolates the effect of equal-magnitude signed weights while leaving
the hidden representations continuous.

\paragraph{Binary weights and binary hidden activations.}

The third model combines binary forward-pass weights with binary
hidden activations. The four hidden affine transformations do not
include separate bias parameters, because each is immediately
followed by batch normalization. For
$\ell=1,\ldots,L-1$, the hidden computation is
\begin{align}
    a^{(\ell)}
        &=
        W^{(\ell)}h^{(\ell-1)},                         \\
    \widetilde a^{(\ell)}
        &=
        \operatorname{BN}^{(\ell)}
        (a^{(\ell)}),                        \\
    h^{(\ell)}
        &=
        \operatorname{sgn}
        (\widetilde a^{(\ell)}).
\end{align}
The binary weights are given by
\begin{equation}
\label{eq:supp-binarynet-forward-weight}
    W^{(\ell)}
    =
    \frac{1}{\sqrt{d_{\ell-1}}}
    \operatorname{sgn}
    \left(\widetilde W^{(\ell)}\right).
\end{equation}
For the width-$32$ networks considered here, every forward-pass
weight therefore has magnitude $1/\sqrt{32}$.

Each batch-normalization operation has one real-valued scale and one
real-valued offset parameter for every hidden unit. The scales are
initialized to $1$, the offsets to $0$, and both are trainable
parameters updated by \texttt{AdamW}. The backward pass through the
sign function uses \eqref{eq:supp-ste-derivative}, and the latent
weights are clipped to $[-1,1]$ after each update as in
\eqref{eq:supp-latent-weight-clipping}.

In the final layer, batch normalization and the sign operation are
omitted. The output layer uses binary forward-pass weights together
with a real-valued output bias and produces unrestricted real-valued
logits. We refer to this model as BinaryNet-style because its
forward-pass weights and hidden activations are binary, while the
latent weights, batch-normalization parameters, output bias, and
surrogate derivatives used during optimization remain real-valued.

Although the hidden affine transformations do not contain explicit
bias parameters, batch normalization introduces learned real-valued
effective thresholds through its scale and offset parameters and its
normalization statistics. These thresholds are not constrained to
the gate-specific relations in \eqref{eq:supp-logic-biases1}-\eqref{eq:supp-logic-biases3}. Moreover,
the binary weight rows remain dense rather than being restricted to
support one or three. This construction follows standard
BinaryNet-style practice \cite{courbariaux2016bnn,umuroglu2017finn} but does not impose the exact logical
structure of the BTF model.

\paragraph{Smooth tanh activations.}

As a smooth relaxation of the BTF activation, the fourth model uses
the hyperbolic tangent in place of ReLU at every hidden layer:
\begin{equation}
\label{eq:supp-tanh-activation}
    h^{(\ell)}
    =
    \tanh
    (a^{(\ell)}),
    \qquad
    \ell=1,\ldots,L-1.
\end{equation}
The weights and biases remain real-valued, and the exact derivative of
$\tanh$ is used during backpropagation. The final layer is linear and
produces unrestricted real-valued logits. Thus, this model differs
from the principal real-valued baseline only through its hidden
activation function. Predictions are obtained by thresholding the
output logits at zero.

\subsubsection{Binary and smooth-threshold network results}

Table~\ref{tab:supp-random-majority-binary-baselines} reports the
terminal train and test bit accuracies.

\begin{table}[t]
\centering
\caption{
Gradient-trained threshold and binary-network baselines on the
random-majority datasets. 
}
\label{tab:supp-random-majority-binary-baselines}
\begin{tabular}{
    l
    S[table-format=3.0]
    S[table-format=3.2]
    S[table-format=2.2]
}
\toprule
Method
& {$N_\mathrm{train}$}
& {Train}
& {Test}
\\
\midrule

\multirow{4}{*}{STE activations}
& 128 & 99.46 & 69.21 \\
& 256 & 97.22 & 72.53 \\
& 384 & 94.79 & 76.06 \\
& 512 & 96.90 & 84.83 \\
\midrule

\multirow{4}{*}{Binary weights + ReLU}
& 128 & 87.13 & 69.45 \\
& 256 & 79.37 & 68.89 \\
& 384 & 72.21 & 66.78 \\
& 512 & 76.32 & 71.47 \\
\midrule

\multirow{4}{*}{Full BinaryNet-style}
& 128 & 77.22 & 66.27 \\
& 256 & 73.88 & 67.55 \\
& 384 & 69.86 & 66.62 \\
& 512 & 69.09 & 66.36 \\
\midrule

\multirow{4}{*}{Tanh activations}
& 128 & 100.00 & 72.67 \\
& 256 & 100.00 & 72.70 \\
& 384 & 100.00 & 73.51 \\
& 512 & 100.00 & 78.60 \\
\bottomrule
\end{tabular}
\end{table}

The network with real-valued weights and binary STE activations does
not interpolate every training set under the tested budget, but its
held-out performance improves substantially as the number of examples
increases. Its test bit accuracy rises from $69.21\%$ at
$N_\mathrm{data}=128$ to $84.83\%$ at
$N_\mathrm{data}=512$. It is the strongest of the four variants at
the largest training-set size.

The binary-weight/ReLU ablation isolates the effect of
equal-magnitude signed weights while retaining continuous hidden
representations. Its training accuracy remains between $72.21\%$ and
$87.13\%$, while its test accuracy lies between $66.78\%$ and
$71.47\%$. Its relatively small train--test gaps therefore primarily
reflect underfitting rather than improved generalization. This result
also shows that dense equal-magnitude signed weights alone do not
reproduce the behavior of the support-constrained BTF model.

The BinaryNet-style model combines binary weights with binary hidden
activations and attains the lowest overall accuracies among the four
variants under this protocol. Its test accuracy remains between
$66.27\%$ and $67.55\%$, and increasing the number of training
examples produces no systematic improvement. Its limited training
accuracy indicates substantial underfitting, although these
experiments do not distinguish optimization difficulty from
limitations of the dense binary hypothesis class.

The tanh network interpolates every training set. Its held-out bit
accuracy increases from $72.67\%$ with $128$ training examples to
$78.60\%$ with $512$ examples. It performs similarly to the
activation-only STE model at $N_\mathrm{data}=256$, but is weaker at
$N_\mathrm{data}=384$ and $512$. Thus, the natural smooth relaxation
fits the observed input--output pairs without attaining the
held-out performance of the activation-only STE model at the larger
training-set sizes.

These experiments separate three properties relevant to logical
networks: strict binary hidden states, equal-magnitude signed weights,
and smooth threshold-like activation. Overall, the two models with
real-valued weights---the activation-only STE model and the tanh
model---provide the strongest gradient-trained results, particularly
at the larger training-set sizes. Nevertheless, under this one-seed,
fixed-budget protocol, none attains the near-perfect held-out
performance obtained by the constraint-trained BTF networks, and none
achieves exact recovery of the target input--output map on the
held-out examples.

\subsubsection{Effect of mini-batch training}
\label{sec:supp-random-majority-minibatch}

The binary-network comparisons above use full-batch updates. We
separately test whether the behavior of the real-valued gradient
baseline arises from processing all training examples in every
update.

The experiment uses $256$ training examples, the
real-valued multilayer perceptron from Appendix~A, five matched
initializations, and a budget of $10^6$ AdamW updates. We compare the
full batch with cyclic and randomly reshuffled mini-batches of sizes
$$
    B\in\{16,32,64,128\}.
$$
In the cyclic schedule, a fixed seed-dependent ordering is repeated
throughout training. In the reshuffled schedule, a new ordering is
sampled after every pass through the training set. We additionally
consider a fixed-subset condition in which the same subset of size
$B$ is used in every update.

The full-batch model reaches $100\%$ median training-bit accuracy and
$74.32\%$ median test-bit accuracy. The cyclic schedules also reach
$100\%$ median accuracy on the complete training pool, with test
accuracies
$$
    74.34\%,\quad
    73.85\%,\quad
    73.97\%,\quad
    73.86\%
$$
for $B=16,32,64,128$, respectively. The corresponding reshuffled
results are
$$
    74.53\%,\quad
    73.65\%,\quad
    73.93\%,\quad
    73.87\%.
$$
Thus, batches as small as $16$ are sufficient to fit the complete
training pool when the sequence of batches eventually presents all
$256$ examples. Random reshuffling provides no clear improvement over
the cyclic schedule.

The fixed-subset experiment gives a different result. Although each
model fits its repeatedly presented subset with $100\%$ bitwise
accuracy, its accuracy on the complete training pool is
$$
    61.69\%,\quad
    68.10\%,\quad
    76.33\%,\quad
    86.19\%
$$
for $B=16,32,64,128$, respectively. The corresponding test accuracies
are
$$
    59.48\%,\quad
    63.43\%,\quad
    68.78\%,\quad
    72.20\%.
$$

The mini-batch results separate the number of examples used in a
single update from the total information presented during training.
Gradient optimization can accumulate information across successive
small batches and interpolate the complete training set. A fixed
subset of the same size, however, is insufficient to determine the
target mapping.

Mini-batching does not materially improve held-out performance.
Across the full, cyclic, and reshuffled conditions, median test-bit
accuracy remains between $73.65\%$ and $74.53\%$, while median
exact-output accuracy remains below $0.25\%$. On this benchmark, the
generalization behavior of the real-valued gradient baseline is
therefore not explained by its use of full-batch updates.

\section{Soft penalties and hard projections in gradient training}
\label{sec:supp-penalty-projected-adamw}

Appendix A reports \texttt{Penalty-AdamW}, which introduces row-norm and
sparsity biases through additional terms in the training objective.
We also tested a \texttt{Projected-AdamW} variant that imposes analogous
conditions directly on the weights after every optimizer update. 

\subsection{\texttt{Projected-AdamW}}

\texttt{Projected-AdamW} first performs an \texttt{AdamW} update on the binary
cross-entropy objective. Each row of every weight matrix is then projected onto
\begin{equation}
\label{eq:supp-projected-adamw-set}
    \mathcal C_{d,k}
    =
    \left\{
        w\in\mathbb R^d:
        \|w\|_0\leq k,
        \quad
        \|w\|_2^2=d
    \right\},
\end{equation}
where $d$ is the fan-in of the layer. We use $k=3$.

For a row $v\in\mathbb R^d$, let $\mathcal T_k(v)$ retain its $k$
largest-magnitude entries and set the remaining entries to zero. For
$\mathcal T_k(v)\neq0$, the update is
\begin{equation}
\label{eq:supp-projected-adamw-map}
    P_{\mathcal C_{d,k}}(v)
    =
    \sqrt{d}\,
    \frac{\mathcal T_k(v)}
         {\|\mathcal T_k(v)\|_2}.
\end{equation}
The implementation clamps the squared norm in the denominator below
by $10^{-12}$ for numerical stability. Bias parameters are not
projected. The projection is applied to every affine layer, including
the output layer.

Unlike \texttt{Penalty-AdamW}, \texttt{Projected-AdamW} does not permit temporarily
dense or incorrectly normalized rows. Support selection is therefore
part of the optimization throughout training: the support of a row
changes only when an excluded entry overtakes an included entry in
magnitude.

\subsection{One-step Rule-30 results}

We evaluated both methods on the one-step Rule-30 problem using the
architecture
$$
    16\longrightarrow32\longrightarrow32\longrightarrow16.
$$
Each run uses full-batch optimization, seed $0$, and $10^6$ updates.
\texttt{Penalty-AdamW} uses learning rate $10^{-3}$, as in Appendix~A, whereas
\texttt{Projected-AdamW} uses learning rate $3\times10^{-4}$. Both use weight
decay $10^{-4}$.

Table~\ref{tab:supp-rule30-penalty-projected} reports terminal
bitwise and exact-output accuracies.

\begin{table}[t]
\centering
\caption{
\texttt{Penalty-AdamW} and \texttt{Projected-AdamW} on one-step Rule 30. 
}
\label{tab:supp-rule30-penalty-projected}
\begin{tabular}{
    l r
    S[table-format=3.2]
    S[table-format=3.2]
}
\toprule
Method
& {$N_\mathrm{data}$}
& {Train bit}
& {Test bit}
\\
\midrule

\multirow{4}{*}{\texttt{Penalty-AdamW}}
& 64  & 100.00 & 89.00  \\
& 96  & 100.00 & 98.76  \\
& 128 & 100.00 & 98.98  \\
& 160 & 100.00 & 100.00 \\
\midrule

\multirow{4}{*}{\texttt{Projected-AdamW}}
& 64  & 62.50 & 52.06 \\
& 96  & 66.21 & 54.86 \\
& 128 & 63.33 & 55.70 \\
& 160 & 60.82 & 55.26 \\
\bottomrule
\end{tabular}
\end{table}

\texttt{Penalty-AdamW} fits every training set exactly. Its test-bit accuracy
increases from $89\%$ at $N_\mathrm{data}=64$ to $100\%$ at
$N_\mathrm{data}=160$, while exact-output accuracy increases from
$18\%$ to $100\%$.

\texttt{Projected-AdamW} does not fit the training data under the tested
protocol. Its terminal training-bit accuracy remains between
$60.8\%$ and $66.2\%$, test-bit accuracy remains between
$52\%$ and $55.7\%$, and exact-output accuracy is zero in every
condition.

Thus, the successful sparsity-enhanced result for one-step Rule 30
does not arise merely from constraining every weight row to be
three-sparse and normalized. The soft formulation permits the model
to fit the observed mapping while the structural penalties continue
to reorganize the weights. The hard projection instead restricts the
available weight geometry after every update and, in this experiment,
prevents even interpolation of the training set.

\subsection{Task dependence}

The effect of \texttt{Penalty-AdamW} depends strongly on the structure of the
target problem. On the random-majority circuits, the penalties do not
produce the transition to near-perfect generalization obtained by
RRR. These tasks require a globally consistent composition of hidden
logical operations across five layers, and a local row-wise sparsity
bias does not identify that complete computation.

For one-step Rule 30, by contrast, each output depends on a small
spatial neighborhood. The sparsity bias is well aligned with this
local structure, and \texttt{Penalty-AdamW} eventually discovers a
generalizing representation. The advantage weakens for two-step Rule
30, where the dependency radius and compositional depth are larger;
under the protocol of the main paper, \texttt{Penalty-AdamW} struggles even to
fit the training data perfectly.

\texttt{Projected-AdamW} also performed comparably poorly on the two-step Rule-30 and random-majority circuit datasets; therefore, we report only its one-step Rule-30 results in the table.

\section{Exact synthesis of the two-bit multiplier}
\label{sec:supp-exact-multiplier}

The multiplier experiment in the Section VI.A demonstrates that training
an initially generic Boolean-threshold network can reveal a compact
and interpretable circuit from the complete multiplication table.
Since the problem has only four input bits, it is also small enough to
permit exact Boolean-circuit synthesis by satisfiability.

For comparison with the learned network, we restrict the
SAT gate library to two-input Boolean functions that are themselves
linearly separable. The four product bits are synthesized jointly in a
single circuit, so intermediate gates may be shared between outputs.
By solving SAT instances with successively larger gate budgets, we
construct a valid shared circuit and certify that no smaller circuit
exists within the prescribed gate library and circuit model.

\subsection{Multiplier specification}

Write the two factors as
\begin{equation}
    A=a_0+2a_1,
    \qquad
    B=b_0+2b_1,
\end{equation}
where
\begin{align*}
    a_0,a_1,b_0,b_1\in\{0,1\}.
\end{align*}
The four product bits are defined by
\begin{equation}
\label{eq:supp-multiplier-definition}
    AB
    =
    p_0+2p_1+4p_2+8p_3.
\end{equation}
The specification is the complete multiplication table containing all
\begin{align*}
    2^4=16
\end{align*}
assignments of the four input bits.

\subsection{Logic-matched gate library}

Consider the bipolar encoding
\begin{align*}
    0\equiv -1,
    \qquad
    1\equiv +1.
\end{align*}
A two-input gate is selected from the signed-majority family
\begin{equation*}
\label{eq:supp-signed-majority-family}
    f(x_1,x_2)
    =
    s_0\,
    \operatorname{Maj}
    \bigl(
        s_c,\,
        s_1x_1,\,
        s_2x_2
    \bigr),
\end{equation*}
where  $s_0,s_c,s_1,s_2\in\{-1,+1\}$. Each member of this family is a Boolean threshold function and is
therefore linearly separable. There are $2^4=16$ signed
parameterizations. Since
\begin{align*}
    -\operatorname{Maj}(u,v,w)
    =
    \operatorname{Maj}(-u,-v,-w),
\end{align*}
the parameterizations occur in equivalent pairs and induce eight
distinct truth tables. These eight functions are precisely
\begin{equation*}
\label{eq:supp-sat-matched-library}
    \mathcal G_{\mathrm{logic}}
    =
    \left\{
        \ell_1\land\ell_2,\,
        \ell_1\lor\ell_2
        \;:\;
        \ell_i\in\{x_i,\neg x_i\}
    \right\}.
\end{equation*}
Explicitly, the library contains
\begin{align*}
    &x_1\land x_2,
    &&x_1\land\neg x_2,
    &&\neg x_1\land x_2,
    &&\neg x_1\land\neg x_2,
    \notag\\
    &x_1\lor x_2,
    &&x_1\lor\neg x_2,
    &&\neg x_1\lor x_2,
    &&\neg x_1\lor\neg x_2.
\end{align*}
In particular, \textsc{Xor} is not available as a primitive gate. For compactness, we use the notation
\begin{align*}
    \operatorname{AND}_{++}(u,v)
        &=u\land v,
    &
    \operatorname{AND}_{+-}(u,v)
        &=u\land\neg v,
    \notag\\
    \operatorname{AND}_{-+}(u,v)
        &=\neg u\land v,
    &
    \operatorname{AND}_{--}(u,v)
        &=\neg u\land\neg v,
    \label{eq:supp-signed-and-notation}
\end{align*}
and analogously
\begin{align*}
    \operatorname{OR}_{++}(u,v)
        &=u\lor v,
    &
    \operatorname{OR}_{+-}(u,v)
        &=u\lor\neg v,
    \notag\\
    \operatorname{OR}_{-+}(u,v)
        &=\neg u\lor v,
    &
    \operatorname{OR}_{--}(u,v)
        &=\neg u\lor\neg v.
\end{align*}

\subsection{Joint circuit model and SAT encoding}

A circuit with $G$ gates is represented by a topologically ordered
sequence
\begin{align*}
    g_0,g_1,\ldots,g_{G-1}.
\end{align*}
Each input of $g_j$ independently selects one signal from the four
primary inputs or the earlier gates
\begin{align*}
    g_0,\ldots,g_{j-1}.
\end{align*}
This ordering enforces acyclicity without prescribing a layered
architecture or a maximum circuit depth. Fan-out is unrestricted, so
an intermediate gate may feed multiple later gates or product outputs.

For a fixed gate budget $G$, Boolean variables encode the type and the
two input sources of every gate. Additional variables encode the value
of every primary input, selected gate input, and gate output on each of
the $16$ truth-table rows. Pairwise exactly-one constraints select one
gate operation and one source for each gate input.

Semantic clauses impose the selected operation separately on every
truth-table row. For example, if $g_j$ is selected to be an
$\operatorname{AND}_{+-}$ gate with inputs $u$ and $v$, then
\begin{equation}
    g_j^{(t)}
    \longleftrightarrow
    \bigl(
        u^{(t)}\land\neg v^{(t)}
    \bigr)
\end{equation}
is imposed for every row $t$. Analogous truth-table clauses implement
the other seven gate types.

All four product bits are included in the same SAT instance. Each
output independently selects a primary input or gate output, and the
selected signal is constrained to agree with the corresponding
multiplier bit:
\begin{equation}
\label{eq:supp-output-constraints}
    \widehat p_r^{(t)}
    =
    p_r^{(t)},
    \qquad
    r=0,1,2,3,
    \qquad
    t=1,\ldots,16.
\end{equation}
Because the four output selectors share the same pool of gates, the
encoding permits intermediate computations to be reused across
outputs.

The gate budget is increased iteratively, beginning with $G=0$. The
first satisfiable budget provides a constructive circuit, while the
preceding unsatisfiable instances certify minimality within
$\mathcal G_{\mathrm{logic}}$, the encoded straight-line circuit
model, and the unit gate-cost measure.

The encoding was implemented using the PySAT toolkit
\cite{imms-sat18}, with \texttt{python-sat} version
\texttt{1.9.dev5} and the \texttt{pysat.solvers.Minisat22} interface
to MiniSat \cite{een2004extensible}. Every decoded circuit was
evaluated independently on all $16$ rows of the multiplier truth
table.

\subsection{Minimum shared circuit}
\label{sec:supp-matched-sat-multiplier}

The instances with gate budgets
\begin{align*}
    G=0,\ldots,7
\end{align*}
were unsatisfiable, whereas the instance with
\begin{align*}
    G=8
\end{align*}
was satisfiable. Consequently,
\begin{equation}
\label{eq:supp-matched-shared-minimum}
    G_{\mathrm{logic}}^\star=8
\end{equation}
is the minimum number of gates within the logic-matched library and
the encoded joint circuit model.

\begin{table}[t]
\centering
\caption{
Joint exact synthesis of the two-bit multiplier with the logic-matched
gate library $\mathcal G_{\mathrm{logic}}$. Intermediate gates may be
shared between outputs. Variables and clauses refer to the first
satisfiable instance.
}
\label{tab:supp-joint-multiplier-matched}
\begin{tabular}{@{}lrrr@{}}
\toprule
Synthesis
& Minimum gates
& Variables
& Clauses
\\
\midrule
Four outputs, shared
& $8$
& $680$
& $9716$
\\
\bottomrule
\end{tabular}
\end{table}

One minimum circuit returned by the solver is
\begin{align}
    g_0 &= \operatorname{AND}_{++}(b_0,a_1),
        \notag\\
    g_1 &= \operatorname{AND}_{++}(b_1,a_1),
        \notag\\
    g_2 &= \operatorname{AND}_{++}(b_0,a_0),
        \notag\\
    g_3 &= \operatorname{AND}_{++}(a_0,b_1),
        \notag\\
    g_4 &= \operatorname{OR}_{++}(g_3,g_0),
        \notag\\
    g_5 &= \operatorname{AND}_{+-}(g_1,g_2),
        \notag\\
    g_6 &= \operatorname{AND}_{++}(g_1,g_2),
        \notag\\
    g_7 &= \operatorname{AND}_{+-}(g_4,g_6),
    \label{eq:supp-matched-shared-circuit}
\end{align}
with outputs
\begin{equation}
\label{eq:supp-matched-shared-outputs}
    p_0=g_2,
    \qquad
    p_1=g_7,
    \qquad
    p_2=g_5,
    \qquad
    p_3=g_6.
\end{equation}

The circuit has a direct partial-product interpretation. First,
\begin{align*}
    g_0=a_1b_0,
    \qquad
    g_3=a_0b_1,
\end{align*}
so
\begin{align*}
    g_4=(a_1b_0)\lor(a_0b_1).
\end{align*}
Moreover,
\begin{align*}
    g_1=a_1b_1,
    \qquad
    g_2=a_0b_0,
\end{align*}
and hence
\begin{align*}
    g_6=g_1g_2
       =a_0a_1b_0b_1
       =(a_1b_0)\land(a_0b_1).
\end{align*}
It follows that
\begin{align}
    p_1
    =g_7
    &=g_4\land\neg g_6
    \notag\\
    &=
    \bigl((a_1b_0)\lor(a_0b_1)\bigr)
    \land
    \neg\bigl((a_1b_0)\land(a_0b_1)\bigr)
    \notag\\
    &=(a_1b_0)\oplus(a_0b_1),
\end{align}
although \textsc{Xor} is not used as a primitive gate.

The remaining outputs are
\begin{align*}
    p_0=g_2=a_0b_0,
\end{align*}
\begin{align*}
    p_2=g_5
       =(a_1b_1)\land\neg(a_0b_0),
\end{align*}
and
\begin{align*}
    p_3=g_6=a_0a_1b_0b_1.
\end{align*}
Direct evaluation verifies
\begin{equation}
    (
        \widehat p_0,
        \widehat p_1,
        \widehat p_2,
        \widehat p_3
    )
    =
    (
        p_0,p_1,p_2,p_3
    )
\end{equation}
for all $16$ input assignments.

\subsection{Comparison with the circuit learned by \texttt{boolearn}}

Joint exact synthesis over $\mathcal G_{\mathrm{logic}}$ gives a
minimum shared circuit with eight gates. The circuit extracted from
the \texttt{boolearn} solution contains nine BTF
nodes. Thus, when the SAT primitives are restricted to the two-input
linearly separable operations associated with the learned circuit,
the learned realization is one node larger than the exact
shared-circuit minimum.

The comparison nevertheless remains representation dependent. SAT
searches directly over discrete gate identities and unrestricted
acyclic wiring, with circuit size fixed explicitly in each instance.
The \texttt{boolearn} experiment begins with an overparameterized
layered BTF network and searches for feasibility under margin,
normalization, and concurrence constraints. Neither the sparse wiring
nor the identities of the logical operations are supplied, and gate
count is not an explicit optimization objective.

The significance of the nine-node result is therefore not that the
training procedure solves a minimum-circuit problem. Rather, a compact
and interpretable multiplier circuit emerges from the projection-based
training problem and lies within one logical element of the exact
minimum obtained by specialized joint synthesis over a matched
linearly separable gate library.

\section{Binary $k$-modes comparison for the MNIST-4 archetypes}
\label{sec:supp-mnist4-kmodes}

The MNIST-4 experiment in Section VI.C produces $16$ binary
archetypes from $1\,024$ binarized images. To provide a task-specific
clustering comparison for this application, we use binary $k$-modes.
This method partitions binary images under Hamming distance and
represents each cluster by its coordinate-wise binary mode.

The comparison distinguishes between the quality of the resulting
partition and the mechanism by which the archetypes are obtained.
Binary $k$-modes directly optimizes the reconstruction of the observed
images using unrestricted binary prototypes. The RRR experiment
instead requires the archetypes to arise through a constrained
Boolean-threshold network. Binary $k$-modes therefore provides a
natural reference for evaluating the reconstruction quality achieved
under the additional network constraints.

\subsection{Data and clustering objective}

The dataset contains $5\,842$ binarized images of the handwritten
digit $4$. Each image is represented as
$$
    x_i\in\{0,1\}^{256},
$$
corresponding to a $16\times16$ pixel array. The first
$N_\mathrm{data}=1\,024$ images are used to fit the prototypes,
matching the data used in the projection-based experiment. The
remaining $N_\mathrm{test}=4\,818$ images are not used for fitting or
restart selection and are used for held-out evaluation.

Let
$$
    p_1,\ldots,p_K\in\{0,1\}^{256},
    \qquad
    K=16,
$$
denote the binary prototypes. For two binary images $x$ and $p$, the
Hamming distance is
\begin{equation}
\label{eq:supp-mnist4-hamming}
    d_H(x,p)
    =
    \sum_{j=1}^{256}
    \mathbf{1}\{x_j\neq p_j\}.
\end{equation}
The training objective is the total distance to the nearest prototype:
\begin{equation}
\label{eq:supp-mnist4-objective}
    J_{\mathrm{train}}
    =
    \sum_{i=1}^{N_\mathrm{data}}
    \min_{1\leq r\leq K}
    d_H(x_i,p_r).
\end{equation}

The objective is minimized by alternating assignment and prototype
updates. For fixed prototypes, each image is assigned to a nearest
prototype. For fixed assignments, each prototype coordinate is
replaced by the majority bit among the images assigned to that
cluster. When the numbers of zeros and ones are equal, the previous
prototype value is retained. The iterations terminate when the
prototypes no longer change or after $200$ updates.

The prototypes are initialized using a Hamming-distance analogue of
$k$-means++. The first prototype is chosen uniformly from the
distinct training images, and each subsequent prototype is selected
with probability proportional to the square of its distance from the
nearest previously selected prototype. Initial prototypes are required
to be distinct. If an assignment produces an empty cluster, the
highest-error image belonging to a cluster with at least two members
is reassigned to the empty cluster before the prototype update.

\subsection{Restart and model-selection protocol}

Because alternating $k$-modes optimization is nonconvex, we perform
$50$ independent runs using seeds
$$
    0,1,\ldots,49.
$$
Every run is fitted using only the $1\,024$ training images. The
reported model is the restart attaining the smallest value of
$J_{\mathrm{train}}$; no held-out statistic is used for model
selection.

The minimum training objective was obtained for seed $19$:
\begin{equation}
\label{eq:supp-mnist4-selected-objective}
    J_{\mathrm{train}}=45\,541.
\end{equation}
This run converged after $26$ alternating updates. Its final solution
contains $16$ nonempty clusters and $16$ distinct binary prototypes.

For a collection of images $\mathcal D$, the binary $k$-modes
reconstruction accuracy is defined by
\begin{equation}
\label{eq:supp-mnist4-accuracy}
    \operatorname{Acc}(\mathcal D)
    =
    100
    \left(
        1-
        \frac{
            \sum_{x\in\mathcal D}
            \min_{1\leq r\leq K} d_H(x,p_r)
        }{
            256|\mathcal D|
        }
    \right).
\end{equation}
This is the fraction of image bits agreeing with the nearest
prototype. It is not a digit-classification accuracy, since all
images belong to the same digit class.

\subsection{Reconstruction results}

For the selected restart, binary $k$-modes achieves a reconstruction
accuracy of $82.6\%$ on the $1\,024$ fitted images. The RRR solution
partitions the same images into $16$ subsets and obtains an in-sample
reconstruction accuracy of $81\%$. Binary $k$-modes therefore
achieves approximately $1.6$ percentage points higher in-sample
reconstruction accuracy.

The difference is consistent with the objectives of the two methods.
Binary $k$-modes directly minimizes the Hamming distortion used for
evaluation, and its prototypes are unrestricted binary vectors. The
RRR archetypes must instead be generated through the
Boolean-threshold network and satisfy its architectural, threshold,
margin, and concurrence constraints.

After the binary $k$-modes restart was selected using training
performance, its prototypes were frozen and evaluated on the
remaining $4\,818$ images. The held-out Hamming objective was
\begin{equation}
\label{eq:supp-mnist4-test-objective}
    J_{\mathrm{test}}=226\,778,
\end{equation}
corresponding to a reconstruction accuracy of $81.6\%$. The
corresponding held-out reconstruction accuracy for the RRR archetypes
is $80.7\%$. Binary $k$-modes therefore achieves approximately $0.9$
percentage points higher held-out reconstruction accuracy.

The Hamming distance from a held-out image to its nearest binary
$k$-modes prototype has mean $47.1$, median $46$, and maximum $134$.
Thus, a held-out $256$-pixel image differs from its nearest learned
mode in approximately $47$ pixels on average.

The binary $k$-modes accuracy decreases from $82.6\%$ on the fitted
images to $81.6\%$ on the held-out images, a reduction of approximately
$1.0$ percentage point. The RRR accuracy decreases from $81\%$ to
$80.7\%$, a reduction of approximately $0.3$ percentage points. These
differences are descriptive results for this fixed split and are not
intended as statistically resolved differences in generalization.

\subsection{Dependence on initialization}

The held-out reconstruction accuracy of binary $k$-modes is stable
across the $50$ initializations. Its minimum, median, and maximum
values are, respectively,
$$
    81.1\%,
    \qquad
    81.6\%,
    \qquad
    81.7\%.
$$
The full range is approximately $0.6$ percentage points. The held-out
accuracy of the training-selected restart is $81.6\%$, equal to the
median at the reported precision. Its held-out performance is
therefore not associated with an unusually favorable initialization.

\subsection{Comparison}

Binary $k$-modes is a specialized baseline for the partitioning
component of the MNIST-4 experiment. On the fitted images, it achieves
$82.6\%$ reconstruction accuracy, compared with $81\%$ for RRR. On
the held-out images, the corresponding accuracies are $81.6\%$ and
$80.7\%$. Binary $k$-modes therefore gives modestly higher
reconstruction accuracy on both portions of the data.

The outputs of the two methods nevertheless have different
interpretations. A $k$-modes prototype is the coordinate-wise mode of
a cluster and need not be realizable as the output of a learned
decoder. In the projection-based experiment, the archetypes arise
from a Boolean-threshold network and are coupled to its internal
representation. The RRR solution therefore supplies both a partition
of the images and a constrained network mechanism that generates the
corresponding binary archetypes.

The comparison does not indicate that RRR improves upon an algorithm
designed specifically to minimize Hamming distortion. Rather, it
shows that the network-generated archetypes attain reconstruction
quality close to that of direct binary clustering while satisfying
the additional structural constraints imposed by the Boolean network.

\subsection{Scope of the comparison}

Only $K=16$ is considered because this matches the number of
archetypes in the projection-based experiment. Alternating
$k$-modes finds a local minimum of the clustering objective and does
not certify global optimality. Pixelwise reconstruction accuracy also
assigns substantial weight to background pixels and does not fully
measure the visual diversity or perceptual quality of the prototypes.

The training values of $82.6\%$ for binary $k$-modes and $81\%$ for
RRR are computed on the same $1\,024$ images. The held-out values of
$81.6\%$ and $80.7\%$, respectively, are computed on the same
remaining $4\,818$ images. The results compare the fraction of
reconstructed image bits under each method's induced partition; they
do not make the mechanisms used to obtain the partitions and
archetypes equivalent.

The binary $k$-modes calculation therefore provides a natural
task-specific clustering reference. Direct optimization of Hamming
distortion gives modestly higher reconstruction accuracy on both the
fitted and held-out images, whereas the projection-based method
obtains comparably representative archetypes as outputs of a learned
Boolean-threshold network.

\bibliographystyle{unsrtnat}
\bibliography{Supp}